\documentclass[preprint,12pt]{elsarticle}
\journal{ArXiv}

\usepackage{amssymb}

\usepackage{soul}
\usepackage{url}
\usepackage[hidelinks]{hyperref}
\usepackage[utf8]{inputenc}
\usepackage[small]{caption}
\usepackage{graphicx}
\usepackage{amsmath}
\usepackage{amsthm}
\usepackage{booktabs}
\usepackage{xfrac}
\usepackage{algorithmic}
\newtheorem{example}{Example}

\usepackage{latexsym}
\usepackage{amsmath}
\usepackage{amssymb}
\usepackage{amstext}
\usepackage{mathtools}
\usepackage{mathrsfs} 
\usepackage{ifsym}
\usepackage{bbm} 
\usepackage{color}
\usepackage[svgnames]{xcolor}
\usepackage{graphics}
\usepackage{graphicx}
\usepackage{tikz} 
\usetikzlibrary{shapes.misc,automata,positioning,arrows}
\usepgflibrary{shapes.arrows} 
\usepackage{dsfont}
\usepackage{url}
\usepackage{setspace}
\usepackage{varioref}
\usepackage{xspace} 
\usepackage{stmaryrd}
\usepackage[ruled,linesnumbered,vlined,scleft]{algorithm2e}
\usepackage{engord}
\usepackage{comment}
\usepackage{pifont}
\usepackage{amsmath}

\usepackage{cmll}

\usepackage[colorinlistoftodos]{todonotes}

\newcommand{\ie}{\mbox{i.e.}}
\newcommand{\eg}{\mbox{e.g.}}
\newcommand{\cf}{\mbox{cf.}}

\newcommand{\wrt}{\mbox{w.r.t.}}
\renewcommand{\st}{\mbox{s.t.}}
\newcommand{\wolog}{\mbox{w.l.o.g.}}

\renewcommand{\ss}{\mathbb{S}}

\newcommand{\tuple}[1]{\langle #1 \rangle}
\newcommand{\defined}{\ensuremath{=_{\text{def}}}}

\newcommand{\overl}[1]{\digamma #1}

\newcommand{\extw}[1]{\langle #1 \rangle}

\newcommand{\vtl}{\vartriangleleft}

\newtheorem{definition}{Definition}
\newtheorem{proposition}{Proposition}
\newtheorem{corollary}{Corollary}
\newtheorem{remark}{Remark}

\newcommand{\Lang}{\ensuremath{\mathcal{L}}}
\newcommand{\calS}{\ensuremath{\mathcal{S}}}

\newcommand{\KB}{\ensuremath{\mathcal{K}}} 
\newcommand{\K}{\KB} 

\newcommand{\limp}{\rightarrow} 
\newcommand{\liff}{\leftrightarrow} 
\newcommand{\lbigand}{\bigwedge}

\newcommand{\limpf}{\rightarrow}

\newcommand{\sat}{\Vdash}

\newcommand{\entails}{\vDash}

\newcommand{\W}{\ensuremath{\mathcal{W}}} 

\newcommand{\km}{\ensuremath{\kappa}}
\newcommand{\kmh}{\ensuremath{\kappa_h}}
\newcommand{\kms}{\ensuremath{\kappa_S}}

\newcommand{\notf}{\neg}
\newcommand{\sentails}{\ensuremath{\leq_s}}

\newcommand{\leqsss}{\ensuremath{\leq}}
\newcommand{\lesss}{\ensuremath {<} }
\newcommand{\po}{\ensuremath{\pi}}
\newcommand{\acc}{\ensuremath{\alpha}}
\newcommand{\conc}{\ensuremath{\gamma}}

\newcommand{\pr}{\mathbb{P}}

\newcommand{\possremaind}{\Bot}

\usepackage[LGR,T1]{fontenc} 

\DeclareSymbolFont{upgreek}{LGR}{cmr}{m}{n}
\SetSymbolFont{upgreek}{bold}{LGR}{cmr}{bx}{n}

\DeclareMathSymbol{\Alpha}{\mathord}{upgreek}{`A}
\DeclareMathSymbol{\Beta}{\mathord}{upgreek}{`B}
\DeclareMathSymbol{\Epsilon}{\mathord}{upgreek}{`E}
\DeclareMathSymbol{\Zeta}{\mathord}{upgreek}{`Z}
\DeclareMathSymbol{\Eta}{\mathord}{upgreek}{`H}
\DeclareMathSymbol{\Iota}{\mathord}{upgreek}{`I}
\DeclareMathSymbol{\Kappa}{\mathord}{upgreek}{`K}
\DeclareMathSymbol{\Mu}{\mathord}{upgreek}{`M}
\DeclareMathSymbol{\Nu}{\mathord}{upgreek}{`N}
\DeclareMathSymbol{\Omicron}{\mathord}{upgreek}{`O}
\DeclareMathSymbol{\Rho}{\mathord}{upgreek}{`R}
\DeclareMathSymbol{\Tau}{\mathord}{upgreek}{`T}
\DeclareMathSymbol{\Chi}{\mathord}{upgreek}{`Q}

\newcommand{\remaind}{\ensuremath{\mathbf{\bot}}}

\usepackage{xspace}
\usepackage{bussproofs}

\newcommand{\lang}{\ensuremath{\mathcal{L}}}

\newcommand{\nd}{\noindent}
\newcommand{\ii}[1]{\mbox{$(#1)$}}

\newcommand{\supent}{\entails}

\newcommand{\Pentmodels}{\supent_{\pr}}

\newcommand{\Pent}{\leq_{\pr}}
\newcommand{\PentP}[1]{\leq_{{#1}}}
\newcommand{\PentPS}{\PentP{{\pr_\Sigma}}}

\newcommand{\PentL}{<_{\pr}}

\newcommand{\Pentlhd}{\lhd_{\pr}}
\newcommand{\PentPlhd}[1]{\lhd_{{#1}}}

\newcommand{\contr}{\div}
\newcommand{\ddiv}{\mathbin{{\div}\mkern-10mu{\div}}}
\newcommand{\contrs}{\ddiv}
\newcommand{\contrk}{\contr_{\km}}
\newcommand{\contrks}{\contrs_{\km}}

\newif\ifmycomment
\mycommenttrue  

\usepackage{lineno}

\begin{document}

\begin{frontmatter}


\title{Belief Change based on Knowledge Measures}

\author[label1]{Umberto Straccia}
\affiliation[label1]{organization={Istituto di Scienza e Tecnologie dell'Informazione, CNR},
            city={Pisa},
            country={Italy}}

\author[label1,label2]{Giovanni Casini}
\affiliation[label2]{organization={CAIR,University of Cape Town},
            city={Cape Town},
            country={South Africa}}


\begin{abstract}
\emph{Knowledge Measures} (KMs)~\cite{Straccia20} aim at quantifying the amount of knowledge/information that a knowledge base carries. On the other hand, \emph{Belief Change} (BC)~\cite{Ferme18,Gardenfors03} is the process of changing beliefs (in our case, in terms of contraction, expansion and revision) taking into account a new piece of knowledge, which possibly may be in contradiction with the current belief.

We propose a new quantitative BC framework that is based on KMs by defining belief change operators that try to minimise, from an information-theoretic point of view, the surprise that the changed belief carries. To this end, we introduce the \emph{principle of minimal surprise}.

In particular, our contributions are \ii{i}  a general information theoretic approach to KMs for which~\cite{Straccia20} is a special case; \ii{ii} KM-based BC operators that satisfy the so-called AGM postulates; and \ii{iii} a characterisation of any BC operator that satisfies the AGM postulates as a KM-based BC operator, \ie, any BC operator satisfying the AGM postulates can be encoded within our quantitative BC framework.
We also introduce quantitative measures that account for the information loss of contraction, information gain of expansion and information change of revision. We also give a succinct look into the problem of \emph{iterated revision}~\cite{DarwichePearl1997}, which deals with the application of a sequence of revision operations in our framework, and also illustrate how one may build from our KM-based contraction operator also one not satisfying the (in)famous \emph{recovery postulate}, by focusing on the so-called \emph{severe withdrawal}~\cite{Rott99} model as illustrative example.
\end{abstract}


\end{frontmatter}

\section{Introduction} \label{intro}
\nd A \emph{knowledge base} (KB) is the main ingredient of a \emph{knowledge-based system} (KBS), whose aim is to use its KB to reason with and make decisions within a specific application domain. A KB essentially represents facts, usually represented via a formal logic, about a specific application domain.

\emph{Knowledge Measures} (KMs)~\cite{Straccia20} aim at quantifying the amount of knowledge/information/surprise that a knowledge base carries and, thus, a KBS contains. To do so, a set of axioms have been defined for classical propositional logic that are believed KMs should satisfy. In~\cite{Straccia20} various contexts have been suggested for which KMs may be useful, among them \emph{Belief Change} (BC)~\cite{Ferme18,Gardenfors03}, which is the process of changing beliefs of a KBS (the facts about the world that a KBS knows) taking into account a new piece of knowledge that possibly may be in contradiction with the current belief.  To do so, various postulates have been defined, \eg~the so-called AGM postulates, that are believed BCs operators should satisfy.

\paragraph{Contribution} We propose here a quantitative BC model that is based on KMs: roughly, we define belief change operators that try to minimise the amount of surprise of the changed belief. To this end, we introduce the \emph{principle of minimal surprise}.
In particular, our contributions are the following: \ii{i} we present a general information theoretic approach to KMs, for which~\cite{Straccia20} is a special case, that will be propaedeutic for BC; 
\ii{ii} we then propose KM-based BC operators that satisfy the AGM postulates;  \ii{iii} we show that any BC operator that satisfies the AGM postulates can be represented as a KM-based BC operator. We also introduce quantitative measures that account for the information loss of contraction, information gain of expansion and information change of revision. 
We also start investigating \emph{Iterated Belief Revision} (IBR)~\cite{DarwichePearl1997} and show that our revision process satisfies three of the four postulates of iterated belief revision~\cite{DarwichePearl1997}, and also illustrate how one may build from our KM-based contraction operator also one not satisfying the (in)famous \emph{recovery postulate}, by focusing on the so-called \emph{severe withdrawal}~\cite{Rott99} model as illustrative example.

We proceed as follows. In the next section, we introduce the background notions we will rely on. In section~\ref{km}, we will recap KMs and propose a more general information-theoretic version. In section~\ref{kmbr} we will show how one may formulate BC using KMs, 
illustrate succinctly IBR and sever withdrawal within our framework. Eventually, section~\ref{concl} summarises our contribution, succinctly addresses related work and proposes topics for future work. Some proof can be found in the appendix.

\section{Background} \label{prelim}


\nd Let $\Sigma=\{p,q,\ldots\}$ be a finite non-empty set of propositional letters (all symbols may have an optional sup-, or sub-script or apex). $\lang_\Sigma = \{\alpha, \beta,\ldots\}$ is the set of propositional formulae based on the set of Boolean operators  $\{ \land, \lor, \neg, \rightarrow, \leftrightarrow \}$ and $\Sigma$. 
A \emph{literal} (denoted $L$) is either a propositional letter or its negation and with $\bar{L}$ we denote $\neg p$ (resp.~$p$) if $L=p$ (resp.~if $L=\neg p$). 
%
%
A \emph{knowledge base} (KB) $\K = \{\phi_1, \ldots, \phi_n\}$ is a finite set of formulae $\phi_i$ and with $\lbigand \K$ we denote the formula $\lbigand_{\phi \in \K} \phi $. In the following, whenever we write $\K$, we consider $\lbigand \K$ instead, unless stated otherwise.
Given a formula $\phi$, with $\Sigma_\phi \subseteq \Sigma$ we denote the set of propositional letters occurring in $\phi$ (we may omit the reference to $\Sigma$ if no ambiguity arises).
We define the \emph{length}  of $\phi$, denoted $|\phi|$, inductively as usual: for $p \in \Sigma$, $|p|=1$, $\neg \phi = 1 + |\phi|$, $|\phi \lor \psi| = |\phi \land \psi| = |\phi \limpf \psi| = |\phi \liff \psi| = 1 + |\phi| + |\psi|$.

%
%

An \emph{interpretation}, or \emph{world},  $w$ \wrt~$\Sigma$ is a set of literals such that all propositional letters in $\Sigma$ occur exactly once in $w$. $\W_\Sigma$ is the set of all worlds \wrt~$\Sigma$.   
%
We may also denote a world $w$ as the concatenation of the literals occurring in $w$ by replacing a negative literal $\neg p$ with $\bar{p}$ (\eg~$\{\neg p, q\}$ may be denoted also as $\bar{p}q$). 
%

With $w\sat \phi$ we indicate that the world $w$ \emph{satisfies} the formula $\phi$, \ie~$w$ is a \emph{model} of $\phi$, which is inductively defined as  usual:
 $w\sat  L$ iff $L \in w$, $w \sat \neg \phi$ iff 
 $w \not \sat \phi$, $w \sat  \phi \land \psi$ iff $w \sat  \phi$ and $w \sat  \psi$, $w \sat  \phi \lor \psi$ iff $w \sat  \phi$ or $w \sat  \psi$, $w \sat  \phi \limpf \psi$ iff $w \sat  \neg \phi \lor \psi$, and eventually $w \sat  \phi \liff \psi$ iff $w \sat ( \phi \limpf \psi) \land (\psi \limpf \phi)$.  
Furthermore, $\phi$ is \emph{satisfiable} (resp.~\emph{unsatisfiable}) if it has (resp.~has no) model. 
We will also use two special symbols $\remaind$ and $\top$:  $\remaind$ is the formula without models, while every world satisfies $\top$.  We impose that $\top$ (resp.~$\remaind$) cannot occur in any other formula different than $\top$ (resp.~$\remaind$) itself.
Given formulae $\phi$ and $\psi$, let 
$[\phi]_{\Sigma}$ be the set of models of $\phi$ \wrt~$\Sigma$; 
define $\phi \entails \psi$ iff $[\phi]_{\Sigma} \subseteq [\psi]_{\Sigma}$, and $\phi$ and $\psi$ \emph{equivalent}, denoted $\phi \equiv \psi$, iff $[\phi]_{\Sigma} = [\psi]_{\Sigma}$.

\begin{remark}\label{finitesets}
In the following, \wolog, for any set of formulae $\mathcal{S} = \{\phi_1, \phi_2,\ldots \}$, we always assume that $\phi_i \not \equiv \phi_j$ for each $i \neq j$. As a consequence, we always have $|\mathcal{S}| \leq 2^{2^{|\Sigma|}}$.
    
\end{remark}

\nd Also, let \ii{i} $\phi \leq \psi$ iff  $\phi \entails \psi$; and
\ii{ii} $\phi {<}  \psi$ iff $\phi\entails\psi$ and $\psi\not\entails \phi$. 
%
%
For $w \in \W_\Sigma$, let 
\begin{equation}\label{ffa}
\overl{w} = \lbigand_{L\in w} L \ ,
\end{equation}
\nd and for $\W \subseteq \W_\Sigma$, let 
\begin{equation}\label{ffb}
\overl{\W}  =  \bigvee_{w \in \W } \overl{w} \ ,
\end{equation}
\nd were $\overl{\emptyset} = \remaind$ and  $\overl{\W_\Sigma} = \top$.
Note that
\begin{equation}\label{ffc}
\phi \equiv \overl{[\phi]_{\Sigma}} \ .
\end{equation}




\begin{remark}  \label{remmods}
    Consider an alphabet $\bar{\Sigma} \subseteq \Sigma$ 
    and a world $\bar{w} \in \W_{\bar{\Sigma}}$ over $\bar{\Sigma}$. Let $m= |\bar{\Sigma}|$, $n = |\Sigma|$ and $k = |\Sigma\setminus \bar{\Sigma}|= n-m$. Of course, there are $2^k$ ways to \emph{extend} the world $\bar{w}$ over $\bar{\Sigma}$ into a world $w$ over $\Sigma$ (for each $p\in \Sigma\setminus \bar{\Sigma}$ add either $p$ or $\neg p$ to $\bar{w}$). With 
    $\extw{\bar{w}}_{\Sigma}$ we will denote the set of such extensions, \ie
    \begin{equation} \label{extw}
        \extw{\bar{w}}_{\Sigma}  =  \{w \in \W_\Sigma \mid \bar{w} \subseteq w \} \ .
    \end{equation}
     \nd Of course, $|\extw{\bar{w}}_{\Sigma}|  = 2^{|\Sigma\setminus \bar{\Sigma}|} = 2^k$ and 
\begin{equation} \label{eqequiv}
\overl{\bar{w}} \equiv \overl{\extw{\bar{w}}_{\Sigma}} \ .
\end{equation}

\end{remark}

\begin{remark} \label{info}
\nd We anticipate (see next section) that we consider a formula $\phi$ as a formal, possibly incomplete, specification of an actual world $w$ and the only constraint we impose on $\phi$ is that $w$ is one of the models of $\phi$, without knowing who it is. Clearly, the more models $\phi$ has the less we know about the actual word $w$.

\end{remark}

\begin{example}[Running example] \label{runexA}
    Consider the following KB about `birds' over $\Sigma = \{b, p, o, f, w \}$, where $b,p,o,f$ and $w$ stand for `bird', `penguin', `ostrich', `flies' and `wings', respectively. The KB
    \[
        \KB  = \{b, 
        b \rightarrow f, 
        p \rightarrow b, 
        o \rightarrow b, \neg (p\land o) \}
    \]
\nd states that we believe that there are birds and that they fly. Also, a penguin is a bird, and an ostrich is a bird, but a penguin cannot be an ostrich and vice-versa. The set of models $ [\KB]_\Sigma$ of $\KB$ over $\Sigma$ is
\[
[\KB]_\Sigma  =  \{
bp\bar{o}fw, 
b\bar{p}ofw, 
b\bar{p}\bar{o}fw,
bp\bar{o}f\bar{w}, 
b\bar{p}of\bar{w}, 
b\bar{p}\bar{o}f\bar{w} 
\} 
\]

\nd It is reasonable to assume that, for any strict subset $\mathcal{S} \subset [\KB]_\Sigma$ (resp.~strict superset $\mathcal{S} \supset [\KB]_\Sigma$), the formula $\overl{\calS}$ encodes more (resp.~less) `information' about the actual world than $\KB$ (see the discussion about Axiom (E) in Section \ref{kmc}).

\end{example}

\section{An information-theoretic approach to knowledge measures} \label{km}

\nd In this section, we first recap succinctly the notion of \emph{knowledge measures} (KMs) according to~\cite{Straccia20} and then propose a generalisation of it based on Shannon's information theory (see, \eg~\cite{Hankerson98}).

\subsection{Knowledge measures recap} \label{kmc}
\nd To start with, given an alphabet $\Sigma$,
a \emph{substitution} $\theta$ is a set 
$\theta = \{p_1 / L_1, \ldots$, $p_{|\Sigma|} / L_{|\Sigma|}\}$ such that each propositional letter in $\Sigma$ occurs exactly once in $ \{p_1, \ldots  p_{|\Sigma|} \}$ as well as in $\{L_1, \ldots  L_{|\Sigma|} \}$. The intuition is that propositional letters may be renamed by literals.  With 
%
%
\ii{i} $w\theta$  we indicate the interpretation obtained from $w$ by replacing every occurrence of $p_i$ in $w$ with $L_i$ (with the convention that double negations are normalised\footnote{$\notf \notf p \mapsto p$.}); and \ii{ii} for a set of interpretations $\W$, with $\W\theta$ we denote the set 
$\W\theta = \{w\theta \mid w \in \W\}$. If $\W_1$ and $\W_2$ are two sets of interpretations,  we write $\W_1 \leqsss_s \W_2$ if there is a substitution $\theta$ such that $\W_1\theta \subseteq \W_2$. If the subset relation is strict, we write $\W_1 \lesss_s \W_2$. Furthermore, we say that a formula $\phi$  \emph{s-entails} a formula $\psi$, denoted $\phi \leqsss_s \psi$, 
if $[\phi]_\Sigma \leqsss_s [\psi]_\Sigma$ holds. We write $\phi \lesss_s \psi$ if $\phi \leqsss_s \psi$ and $\psi \not \leqsss_s \phi$.\footnote{We use the same relation symbol whether ranging over sets of interpretations or formulae. We will disambiguate it whenever required.} 
Note that if $\phi \models \psi$ then $\phi \sentails \psi$, but not vice-versa~\cite{Straccia20}.
We also write $\phi \equiv_s \psi$ if $\phi \leqsss_s \psi$ and $\psi \leqsss_s \phi$ hold.

\begin{example}[\cite{Straccia20}] \label{ent}
It is easily verified that  using substitution $\theta = \{p/\bar{q}, q/p \}$, $p \leqsss_s \bar{q}$, but 
$p \not \leq \bar{q}$, \ie~$p \not \models \bar{q}$. Similarly, it is easily verified that $\bar{q} \leqsss_s p$ and, thus, $p \equiv_s \bar{q}$.

\end{example}

\nd We anticipate that the intuition behind s-entailment is to account for the idea that all propositional literals carry the same `amount of information'. For instance, in Example~\ref{ent} above, $p$ and $\bar{q}$ will carry the same amount of information as $p \equiv_s \bar{q}$~\cite{Straccia20}.

Now, there are three simple principles KMs rely on (\cf~Remark~\ref{info}).

\begin{enumerate}
\item A formula $\phi$ is considered as a formal specification of the actual world, which is one of the models of $\phi$.

\item The more $\phi$ entails the more $\phi$ knows about the actual world.

\item The more models a formula has, the more uncertain we are about which is the actual world, which in turn means the less we know about the actual world.

\end{enumerate}

\nd These principles led to the following four axioms for KMs~\cite{Straccia20}.
\nd A  KM $\km$ is a function mapping formulae into 
$\mathbb{R}^+\cup\{0,\infty\}$ satisfying:
\begin{description}
\item[(T):]  $\km(\top) = 0$ and  $\km(\remaind) = \infty$;
\item[(E):]  if~$\phi \vtl_s \psi$ then~$\km(\psi) \vtl \km(\phi)$, for $\vtl \in \{\leq,<\}$. 
\item[(L):]  $\km(\phi)$ does depend on $\Sigma_\phi$ only.
\item[(M):]  if~$\phi$ is satisfiable then 
\begin{enumerate}
    \item $0 \leq \km(\phi) \leq | \Sigma_\phi|$; and 
    \item if~$|[\phi]_{\Sigma_\phi}| =  1$ then~$\km(\phi) = |\Sigma_\phi|$.
\end{enumerate}
\end{description}

\nd Let us briefly explain the above (TELM) axioms.
%

Concerning axiom (E), assume $\phi \models \psi$. Then $\phi$ has fewer models than $\psi$, which means that there is less uncertainty about which is the actual model for $\phi$ compared to $\psi$. Moreover, whatever is entailed by $\psi$ is also entailed by $\phi$. Combining the two facts means that $\phi$ represents more information about what is the actual world than $\psi$. Concerning s-entailment, \eg~for $\phi:=p$ and $\psi:=q$ neither $\phi \models \psi$ nor $\psi \models \phi$ hold. But, according to axiom (E), we have that $\km(\phi) = \km(\psi)$ instead, \ie~$p$ and $q$ carry the same amount of knowledge: that is, according to~\cite{Straccia20}, a KM is insensitive to symbol names, \ie~a symbol $p$ carries as much knowledge as another symbol $q$. 

Concerning  axiom (T), note that if $\models \phi$  then for every formula $\psi$, we have that $\psi \models \phi$ and, thus, by axiom (E) $\km(\psi) \geq \km(\phi)$ has to hold, \ie~$\km(\phi)$ has to be as small as possible, which motivates $\km(\phi)=0$.
Analogously, if $\phi$ is unsatisfiable then
for every formula $\psi$,  we  have that $\phi \models \psi$ and, thus,  by axiom (E) $\km(\phi) \geq \km(\psi)$ has to hold. That is, $\km(\phi)$ has to be as large as possible, which motivates $\km(\phi)  = \infty$.

Concerning axiom (L), this axiom says that, to what concerns KMs, we may restrict our attention to $\Sigma_\phi$, \ie~the set of all propositional letters occurring in a formula $\phi$ and, thus, symbols not occurring in $\phi$ do not contribute to represent additional information. Note, however, that such an assumption may not be considered valid \eg~under the \emph{Closed World Assumption} (CWA)~\cite{Minker82}, where the truth of a letter not occurring in a formula is always false.

Concerning axiom (M), this axiom tells us that KMs are bounded in the sense that a satisfiable formula may not represent more information than the number of symbols it relies on and the bound is reached only if the formula exactly describes the actual world.

It has been shown in~\cite{Straccia20} that a KM satisfying the axioms above is
\begin{equation} \label{kmhs}
\boxed{ 
\kmh(\phi) =  |\Sigma_\phi| - \log_2 |[\phi]_{\Sigma_\phi}| 
} 
\end{equation}
\nd where $\kmh(\phi) = \infty$, if $\phi$ is not satisfiable.

\begin{example}[Running example cont.] \label{runexB}
    Consider the KB $\KB$ in Example~\ref{runexA}. 
Then  we have 
     $\Sigma_\KB = \{b,p,o,f\}$, and 
     $[\KB]_{\Sigma_\KB}  =  \{bp\bar{o}f,b\bar{p}of, b\bar{p}\bar{o}f\}$ 
     and, thus,
\[
 \kmh(\KB) = 4 - \log_2 3 \approx 2.415 \ . 
\]


\end{example}


\subsection{An information-theoretic approach to knowledge measures} \label{kmp}

\nd We next show how we are going to generalise the above notion of KMs. Specifically, we revert to the well-known notion of \emph{information theory} (see, \eg~\cite{Hankerson98}), where the core idea is that the `informational value' of a communicated message (in our case a formula) depends on the degree to which the content of the message is \emph{surprising}. If a highly likely event occurs, \ie~the probability of a formula is high, the message carries very little information. On the other hand, if a highly unlikely event occurs, the message is much more informative.

As before, we still assume that one of the worlds is the actual world and we may be uncertain about which one it is exactly. However, now a probability distribution over the worlds is given. Of course, the assumption underlying the characterisation of KM proposed in~\cite{Straccia20} is that of a uniform probability distribution as each world is equally plausible.

Formally, consider an alphabet $\bar{\Sigma} \subseteq \Sigma$. We will assume that a probability distribution $\pr_{\bar{\Sigma}}$ over the worlds in $\W_{\bar{\Sigma}}$, $\pr_{\bar{\Sigma}} \colon \W_{\bar{\Sigma}} \mapsto [0,1]$ is given. The intended meaning of $\pr_{\bar{\Sigma}}(w)$ is to indicate how probable it is that $w$ is indeed the actual world.
A world $w$ is \emph{possible}, if it has non-zero probability, \ie~$\pr_{\bar{\Sigma}}(w) >0$.
For ease of presentation, if no ambiguity arises, we may also denote $\pr_{\bar{\Sigma}}(w)=n$ as $w^{n}$ and, thus, \eg~may simply write $w^{0.6}$ in place of $\pr_{\bar{\Sigma}}(w)=0.6$.
So, we may also write $\pr_{\bar{\Sigma}}$ as a set 
$\{w_1^{p_1}, \ldots,  w_k^{p_k}\}$, where $\pr_{\bar{\Sigma}}(w_i) = p_i$, and the convention that worlds not occurring in that set have $0$ probability.
The \emph{uniform probability distribution} is, of course, defined as $\pr^u_{\bar{\Sigma}}(\bar{w}) = 1/|\W_{\bar{\Sigma}}|$. 

Now, the probability that one of the models of $\phi$ is indeed the actual world, called the \emph{probability of a formula} $\phi$ \wrt~the alphabet $\bar{\Sigma}$, denoted $\pr_{\bar{\Sigma}}(\phi)$, is defined as 
%
\begin{eqnarray} \label{prphi}
\pr_{\bar{\Sigma}}(\phi) & = & \sum_{\bar{w}\in [\phi]_{\bar{\Sigma}} } \pr_{\bar{\Sigma}}(\bar{w}) \ .
\end{eqnarray}
\nd Note that $\pr_{\bar{\Sigma}}(\top)=1$, $\pr_{\bar{\Sigma}}(\remaind)=0$ and, if $\pr_{\bar{\Sigma}}$ is the uniform probability distribution, then $\pr_{\bar{\Sigma}}(\phi) = |[\phi]_{\bar{\Sigma}}|/|\W_{\bar{\Sigma}}|$.
We may also simply write $\phi^p$ in place of $\pr_{\bar{\Sigma}}(\phi) = p$ if clear from the context. With $[\phi]_{\bar{\Sigma}}^{+}$ we denote the \emph{set of possible models} of $\phi$, \ie~the set of models of $\phi$ that have non-zero probability: \ie
\[
[\phi]_{\bar{\Sigma}}^{+}=\{w\in \W_{\bar{\Sigma}} \mid w\models \phi \text{ and } \pr_{\bar{\Sigma}}(w) > 0\} \ .
\]

\begin{remark}\label{remprolog}
Please note that two formulae $\phi$ and $\psi$, aimed at describing information about the actual world (\ie~one of the worlds in $\W_{\Sigma}$), share the \emph{same} probability distribution $\pr_{\Sigma}$ over worlds. The latter will induce a probability $\pr_{\Sigma}(\phi)$ (resp.~$\pr_{\Sigma}(\psi)$) indicating how likely one of the models of $\phi$  (resp.~$\psi$) is indeed the actual world. This will then allow us to compare the amount of information encoded by $\phi$ and $\psi$ on equal terms, \ie~based on the same probability distribution.

Note also that a formula $\phi$  does not impose restrictions on $\pr_{\Sigma}$,  but we only expect that $\phi$ is `compatible' with the scenario depicted by $\pr_{\Sigma}$:  that is, one of the models of $\phi$ has a non-zero probability. Moreover, what is considered an impossible world according to $\pr_\Sigma$ has to be false (see $\pr$-entailment later on). 

This is different from numerous probabilistic logics proposed in the literature, 
(see \eg~\cite{Halpern17}) in which typically a probabilistic KB is a formal specification about probabilities and/or subjective probabilities of worlds. We do not consider this scenario here and leave it for future work.

\end{remark}

\nd Next, we are also going to use some additional notions. Specifically, the \emph{extension} $\pr_{\Sigma}$ of $\pr_{\bar{\Sigma}}$ to $\Sigma$ is defined in the following way: consider
$\bar{w} \in \W_{\bar{\Sigma}}$ and an extension $w \in \extw{\bar{w}}_\Sigma$ of $\bar{w}$ (see  Remark~\ref{remmods}).
Then, 
\begin{eqnarray}\label{sigmapB}
\pr_{\Sigma}(w) & =  & \frac{1}{|\extw{\bar{w}}_\Sigma|} \cdot \pr_{\bar{\Sigma}}(\bar{w}) \ 
\end{eqnarray}
\nd so that $\pr_{\bar{\Sigma}}$ is the \emph{marginal distribution}  of $\pr_{\Sigma}$,
\ie
\begin{eqnarray}\label{sigmapbisB}
\pr_{\bar{\Sigma}}(\bar{w}) & =  & \sum_{w \in \extw{\bar{w}}_\Sigma} \pr_\Sigma(w) \ .
\end{eqnarray}

\begin{proposition} \label{prsigmapB}
For a formula $\phi$, $\Sigma_\phi \subseteq \bar{\Sigma} \subseteq \Sigma$ and a probability distribution $\pr_{\bar{\Sigma}}$, we have that
$\pr_\Sigma(\phi)  =   \pr_{\bar{\Sigma}}(\phi)$,
where $\pr_\Sigma$ is the extension of $\pr_{\bar{\Sigma}}$.
\end{proposition}

\nd In particular, by Proposition~\ref{prsigmapB}, for any formula $\phi$, 
\begin{eqnarray} \label{prsphiB}
\pr_{\Sigma}(\phi) & = & \pr_{\Sigma_\phi}(\phi) \ ,
\end{eqnarray}
\nd where $\pr_\Sigma$ is the extension of $\pr_{\Sigma_\phi}$.
In summary, extension and marginalisation are two operations that allow us to `extend' or to `restrict' probability distributions over alphabets $\bar{\Sigma} \subseteq \Sigma$.



\begin{example}[Running example cont.]\label{runexC}
Consider Example~\ref{runexA}. Let us assume that we have the following probability distribution $\pr_\Sigma$ \wrt~$\Sigma = \{b,p,o,f,w\}$~\footnote{If a world is not listed, then its probability is $0$.}:
\begin{eqnarray*}
    \pr_\Sigma & \defined & 
    \{bp\bar{o}fw^{0.1}, b\bar{p}ofw^{0.1}, b\bar{p}\bar{o}fw^{0.15}, bp\bar{o}\bar{f}w^{0.15}, \\ 
    && b\bar{p}o\bar{f}w^{0.2}, b\bar{p}\bar{o}\bar{f}w^{0.2}, \bar{b}p\bar{o}fw^{0.07} ,  \bar{b}p\bar{o}f\bar{w}^{0.03}  \} \ . 
\end{eqnarray*}
\nd Then, the marginalisation of $\pr_\Sigma$ to $\Sigma_\KB$ is
  \begin{eqnarray*}
    \pr_{\Sigma_\KB} & \defined & 
    \{bp\bar{o}f^{0.1}, b\bar{p}of^{0.1}, b\bar{p}\bar{o}f^{0.15}, bp\bar{o}\bar{f}^{0.15}, \\
    &&   b\bar{p}o\bar{f}^{0.2}, b\bar{p}\bar{o}\bar{f}^{0.2}, \bar{b}p\bar{o}f^{0.1} \} \ .
  \end{eqnarray*}
\nd Therefore, \wrt~$\pr_{\Sigma_\KB}$ we have that the following probabilities of formulae:
  \[
  \begin{array}{l}
    b^{0.9}, p^{0.35} , o^{0.3}, f^{0.45}, w^{1.0}, (p\land o)^{0.0}, \\
    (b\rightarrow f)^{0.45}, (p\rightarrow b)^{0.9}, (o\rightarrow b)^{1.0}
  \end{array}
  \]
  \nd and, in particular, 
  $\pr_{\Sigma}(\KB) = \pr_{\Sigma_\KB}(\KB) = 0.35$.
  That is, the probability that one of the models of $\KB$ is indeed the actual world is $0.35$ 
  \wrt~$\pr_\Sigma$.
\end{example}

\nd As next, we extend the notion of entailment to the case in which we have a probability distribution over worlds.
To do so, we define the following supraclassical monotonic entailment relation: consider formulae $\phi$ and $\psi$ \wrt~$\Sigma$, a probability distribution $\pr$ over $\W_{\Sigma}$. Let us define the formula
\begin{equation}
  \pr0 = \overl{{\{w \mid \pr(w)=0, w \in \W_{\Sigma} \}}}   
\end{equation}
 
\nd as the disjunction of all worlds having $0$ probability. So, $\pr(\pr0)=0$ and, thus, $\pr(\neg \pr0)=1$.

Then the notion of $\phi$ \emph{$\pr$-entails} $\psi$, denoted $\phi \Pent \psi$ or also $\phi \Pentmodels \psi$, is defined as the case in which all non-zero $\pr$-probability models of $\phi$ are $\psi$ models: \ie
\begin{equation}
\phi \Pent \psi \text{ iff } \phi \land \neg \pr0 \entails \psi,    
\end{equation}
%
\nd That is, $\phi \Pent \psi$ iff $[\phi]^+_\Sigma  \subseteq [\psi]_\Sigma$. We also define \ii{i} $\phi \PentL  \psi$ iff $\phi\Pent\psi$ and $\psi \not\Pent \phi$; and \ii{ii} $\phi\equiv_{\pr} \psi$ ($\phi$ and $\psi$ are \emph{$\pr$-equivalent}) iff $\phi\Pent \psi$ and $\psi\Pent \phi$.  We say that $\phi$ is \emph{$\pr$-consistent}, 
or also \emph{$\pr$-satisfiable} iff $\phi \not \Pent \bot$.





\begin{remark} \label{nonz}
Note that $\phi \land \neg \pr0$ has a model only iff $\phi$ has a non-zero probability model, \ie~$\phi$ is $\pr$-satisfiable. Therefore,
$\phi\PentL \psi$ tell us that all non-zero $\pr$-probability models of $\phi$ are $\psi$ models and that there is a non-zero $\pr$-probability model of $\psi$ that is not a model of $\phi$. 
\end{remark}

\begin{remark}[Classical case]\label{claskm}
Classical propositional logic can be obtained as a special case if, \eg~the  uniform probability distribution $\pr^u$ is considered.  In that case, $\phi \PentP{\pr^u}\psi$  iff $\phi \entails \psi$ holds. 

However, this result does not generalise to any probability distribution $\pr$: 
while $\phi \models \alpha$ implies $\phi \Pent \alpha$, the opposite is not true. In fact, \eg, for 
$\phi\defined a$ and $\pr(a\bar{b})=0$, we have $\phi \land \neg \pr0 \defined a \land \neg (a \land \neg b) \models b$ and thus, $\phi \Pent b$, but $\phi \not \entails b$.    
\end{remark}


\nd The following Propositions can easily be shown.

\begin{proposition}\label{propextconf}
Consider formulae $\phi$ and $\psi$, $\Sigma_\phi \cup \Sigma_\psi \subseteq \bar{\Sigma} \subseteq \Sigma$,
a probability distribution $\pr$ over $\bar{\Sigma}$, and the extension $\pr_{\Sigma}$  of  $\pr$  to $\Sigma$. Then for $\lhd\in\{\leq,<\}$, $\phi \Pentlhd \psi$ iff $\phi \PentPlhd{\pr_{\Sigma}} \psi$.
\end{proposition}

\nd The following Proposition will be used extensively later and can easily be shown.

\begin{proposition} \label{probent}
    Consider formulae $\phi$ and $\psi$ \wrt~$\Sigma$, and a probability distribution $\pr$ over $\Sigma$.
    If $\phi \Pentlhd \psi$ then $\pr(\phi) \lhd \pr(\psi)$, for  $\lhd\in\{\leq,<\}$.
\end{proposition}

\nd Moreover, the inverse of Proposition \ref{probent} does not hold.

\begin{example}[Running example cont.] \label{runexD}
    Consider Example~\ref{runexC} and $\pr=\pr_{\Sigma_\KB}$. 
It is easily verified that 
$\KB \PentL f \land (\neg b \rightarrow (p \lor o \lor f))$.
%
Moreover, note that  $\pr_{\Sigma_\KB}((\neg b \rightarrow (p \lor o \lor f))) = 1$ as 
$\bar{b}\bar{p}\bar{o}\bar{f}^{0.0}$ holds. It follows that
$\pr_{\Sigma_\KB}(f \land (\neg b \rightarrow (p \lor o \lor f))) = 1 > 0.35 = \pr_{\Sigma_\KB}(\KB)$ in accordance with Proposition~\ref{probent}.

Moreover, $\pr_{\Sigma_\KB}(b \land p \land \neg o \land f) = 0.1 < 0.15 = \pr_{\Sigma_\KB}(b \land \neg p \land \neg o \land f)$, but
$b \land p \land \neg o \land f  \not \PentL b \land \neg p \land \neg o \land f$ and, thus, the converse of Proposition~\ref{probent} does not hold.
%
\end{example}

\nd We conclude this part by defining the notion of independent formulae: namely, consider two formulae $\phi$ and $\psi$, and a probability distribution $\pr$ over $\Sigma$, then
we say that $\phi$ and $\psi$ are \emph{$\pr$-independent} iff 
\[
\pr(\phi\land\psi)  =  \pr(\phi) \cdot \pr(\psi) \ .
\]

\begin{example}[Running example cont.] \label{runexE}
Consider Example~\ref{runexC}. Then $b$ and $w$ are $\pr_\Sigma$-independent, while $b$ and $f$ are not.
\end{example}

\nd We are ready now to present the generalisation for KMs: given a probability distribution $\pr$ over $\Sigma$, a KM $\km$ is a total function mapping formulae over $\Sigma$ into  $\mathbb{R}^+\cup\{0,\infty\}$ satisfying the following three axioms:

\begin{description}

    \item[{\bf (KM1)}]  $\km(\top)=0$ and $\km(\remaind) = \infty$;

    \item[{\bf (KM2)}] $\km$ is monotone non-increasing, \ie~if $\pr(\phi) \lhd \pr(\psi)$ then  $\km(\psi) \lhd \km(\phi)$, for $\lhd\in\{\leq,<\}$;

    \item[{\bf (KM3)}] $\km$ is additive, \ie~if $\phi$ and $\psi$ are  $\pr$-independent, then $\km(\phi\land\psi)=\km(\phi)+\km(\psi)$.
    
\end{description}

\nd Let us shortly explain the above axioms. Of course, axiom (KM1) corresponds to axiom (T), which we already commented. 
Axiom (KM3) encodes a sort of model independence: in this case, we may join $\phi$ and $\psi$ without any loss of information \wrt~the sum of both. Eventually, concerning (KM2), at first note the following: 
if $\phi \Pentlhd \psi$, then by Proposition~\ref{probent}, 
$\pr(\phi) \lhd \pr(\psi)$ and, thus, by (KM2), $\km(\psi) \lhd  \km(\phi)$. So, this case resembles axiom (E). However, as $\pr(\phi)\lhd \pr(\psi)$ does not imply $\phi\Pentlhd \psi$, axiom (KM2) is weaker than (E), and it is more appropriate to the present framework since now we deal also with non-uniform probability distributions.

From what we have said above, the following is immediate.
\begin{proposition} \label{prKM}
Consider formulae $\phi$ and $\psi$ \wrt~$\Sigma$, a probability distribution $\pr$ over $\Sigma$, if $\phi\Pentlhd \psi$ then $\km(\psi) \lhd  \km(\phi)$, for $\lhd\in\{\leq,<\}$.
\end{proposition}

\begin{example} (No surprise)\label{renosur}
Consider $\pr$ over $\Sigma=\{a,b\}$ with $\pr(ab) = 1$ (all other probabilities are $0$). Then
$\top \models_\pr a\land b$ and, thus, by (KM1) and (KM2) we get $0 =\km(\top) \geq \km(a \land b) \geq 0$, \ie~$\km(a \land b)=0$. That is, under $\pr$, $a \land b$ is no surprise and, thus, does not provide us with any information. On the other hand, according to~\cite{Straccia20}, under uniform probability distributions, we get $\kmh(a \land b) = 2$.
\end{example}

\nd As known from information theory~\cite{Hankerson98},
a function satisfying (KM1)-(KM3) can be defined in the following way:
given a probability distribution $\pr$ over $\Sigma$, the \emph{Shannon KM}, or simply the \emph{S-KM} of a formula $\phi$, denoted $\kms(\phi)$, is defined as 
\begin{equation} \label{khconc}
\boxed{ \kms(\phi) =  - \log_2 \pr(\phi)   } \ ,
\end{equation}
\nd where we postulate $\log_2 0 = -\infty$. Note that the probability distribution $\pr$ is a parameter of $\kms$, which we omit to ease the presentation. In case of ambiguity, we will write $\kms(\phi, \pr)$ instead. Noteworthy, the function $\kms$ is the unique solution satisfying (KM1) - (KM3), up to a multiplying constant, which is a direct consequence of the uniqueness result of the \emph{entropy} function proved by Shannon in~\cite{Shannon48}.
That is,
\begin{proposition} \label{kmsprop}
 A KM $\km$ satisfying (KM1) - (KM3) is of the form 
 $\frac{1}{\log_2 b} \cdot \kms$,
 \nd for a constant $b>1$, 
 \ie~for a formula $\phi$ \wrt~$\Sigma$ and a probability distribution $\pr$ over $\Sigma$,
 $\km(\phi) = -\log_b \pr(\phi) = \frac{1}{\log_2 b} \cdot \kms (\phi)$.
 Also, $\kms$ satisfies (KM1) - (KM3).
\end{proposition}

\begin{remark}
By Eq.~\ref{prsphiB}, it suffices to consider a probability distribution over $\Sigma_\phi$ to determine $\kms(\phi)$.
\end{remark}


\nd Now, it can easily be shown that Straccia's KM $\kmh$~\cite{Straccia20} is a special case of $\kms$.

\begin{proposition} \label{kmsclassical}
For any formula $\phi$, given the uniform probability distribution $\pr^u_{\Sigma}$ over $\Sigma$,
$\kms(\phi) = \kmh(\phi)$.
\end{proposition}

\nd An immediate consequence of Proposition~\ref{kmsclassical} is also:
\begin{proposition} \label{kmstelm}
 Under uniform probability distribution, and, thus, classical logic, $\kms$ satisfies the (TELM) axioms.
\end{proposition}

\begin{example}[Running example cont.]\label{runexF}
Consider Example~\ref{runexC}. Then
    \[
    \kms(\KB) = - \log_2 0.35 = 1.515 \ .
    \]
    \nd By referring to Example~\ref{runexD}, we also have that
$\KB \PentL f$ and, thus, $1.515 \approx \kms(\KB) > \kms(f)= - \log_2 0.45 \approx 1.152$, that is, $\KB$ has more knowledge about the actual world than $f$, as expected by Proposition~\ref{prKM}.
\end{example}


\begin{remark}\label{remotherm}
In~\cite{Straccia20} some other related measures have been introduced: namely,
$\bar{\acc}(\phi)  =  \km(\phi)/|\Sigma_\phi|$ (accuracy),
$\bar{\conc}(\phi)   =  \km(\phi)/|\phi|$ (conciseness), and
$\po(\phi)  =  \arg\max_{\{\psi \mid \psi \equiv \phi\}} \acc(\psi) \cdot \conc(\psi)$ (Pareto optimality),
 where $|\phi|$ is the length of $\phi$~\cite{Straccia20}. The first one defines how precise a KB is in describing the actual world, the second one defines how succinct a KB is \wrt~the knowledge it represents, while the last one establishes when we may not increase accuracy without decreasing conciseness (or vice-versa). 
These measures can be extended to S-KMs as well by replacing $\equiv$ with $\equiv_{\pr}$, but we will not address it here.\footnote{Note, however, that these measures do not satisfy the KM axioms by Proposition~\ref{kmsprop}. For instance, accuracy does not satisfy (KM2).}

\end{remark}

\section{KMs-based belief change} \label{kmbr}

 \nd We now show how we may apply KMs to \emph{Belief Change} (BC)~\cite{Ferme18,Gardenfors03}. We recap that BC is the research area dedicated to the formal definition of how a rational agent should manage its own beliefs while facing new pieces of information and aiming at preserving the overall logical consistency of the KB. The two main BC operations that have been investigated are \emph{revision} and \emph{contraction}: the former deals with an agent incorporating in its own KB a new piece of information, adjusting the content of the KB to avoid the risk of creating new inconsistencies; the latter analyses an agent that has to readjust the KB so that a specific piece of information is not derivable anymore, for example because it has been informed that the source is not reliable.  Such two kinds of operations have been characterised from a logical point of view by associating a set of formal postulates to each of them. The AGM framework \cite{AlchourronEtAl1985} is the most popular proposal.

Various axiomatisations have been defined, \eg~the AGM postulates, that are believed BC operators should satisfy and, to date, a multitude of different ways for performing these operations have been proposed (see, \eg~\cite{Ferme18}).  In what follows we are going to consider the AGM approach as the basic framework and formalise quantitative BC operators employing KMs.


Our aim here is to propose KM-based BC operators that satisfy the AGM postulates and eventually to show that any BC operator that satisfies the AGM postulates can be represented as a KM-based BC operator.

It is well-known that in the AGM framework, there are three types of BCs. Given a belief $\phi$ (of an agent), in
\begin{enumerate}
    \item \emph{contraction},  a sentence $\alpha$ is removed, \ie, a belief $\phi$ is replaced by another belief $\phi \contr \alpha$ that is a consequence of $\phi$ not entailing $\alpha$;

    \item \emph{expansion}, a sentence $\alpha$ is added to $\phi$ and nothing is removed, \ie~$\phi$ is replaced with $\phi+\alpha  = \phi \land \alpha$;  

\item \emph{revision}, a sentence $\alpha$ is added to $\phi$, and at the same time other sentences are removed if this is needed to ensure that the resulting belief  $\phi \star \alpha$ is consistent.
    
\end{enumerate}

 \nd A typical way to achieve a revision operator is to rely on the \emph{Levi identity} allowing to define revision in terms of contraction. Alternatively, one may rely on the \emph{Harper identity} to define contraction in terms of revision: that is (see also~\cite{Caridroit17}), 

\begin{description}
    \item[Levi identity:] $\phi \star \alpha = (\phi \contr \neg \alpha) \land \alpha$;
    \item[Harper identity:] $\phi \contr \neg \alpha = \phi \lor (\phi \star  \alpha)$.
\end{description}

\nd We will follow the former path by defining a KM-based contraction operator and then use the Levi identity to obtain a KM-based revision operator.

To start with, consider a probability distribution $\pr$ over $\Sigma$. Let us stress the fact that $\pr$ is a parameter of our setting and that different distributions may give rise to different BC operators. However, as we will show,  all of them will satisfy the AGM postulates. Moreover, please note that any belief $\phi$ and $\psi$  share the same probability distribution $\pr$ (\cf~Remark~\ref{remprolog}).
%
Additionally, for the rest of the paper, we will also consider $\pr$-satisfiable formulae only, \ie~our belief must be compatible \wrt~$\pr$.

\subsection{Contraction} \label{kmbragm}

\nd A total function $\contr\colon\Lang_\Sigma\mapsto\Lang_\Sigma$ is a \emph{contraction operation} if it satisfies the following postulates (for ease, we write $\phi^{\contr}_{\alpha}$ in place of $\phi \contr \alpha$), which are the $\pr$-entailment analogues of those defined for classical propositional logic as in~\cite{Caridroit17}:
%
%
{
\[
\begin{array}{lll}
(\contr 1)    &  \phi\Pent \phi^{\contr}_{\alpha} & \text{(inclusion)}\\
(\contr 2)    & \text{If }\phi\not\Pent \alpha\text{, then }\phi^{\contr}_{\alpha}\equiv_{\pr} \phi & \text{(vacuity)}\\
(\contr 3) & \text{If } \top \not\Pent \alpha\text{, then }\phi^{\contr}_{\alpha}\not\Pent \alpha & \text{(success)}\\
(\contr 4) & \text{If }\alpha\equiv_{\pr} \beta\text{, then }\phi^{\contr}_{\alpha}\equiv_{\pr} \phi^{\contr}_{\beta} & \text{(extensionality)}\\
(\contr 5) & \phi^{\contr}_{\alpha}\land\alpha\Pent \phi & \text{(recovery)}\\
(\contr 6) & \phi^{\contr}_{\alpha\land \beta}\Pent\phi^{\contr}_{\alpha}\lor\phi^{\contr}_{\beta} & \text{(conjunctive overlap)}\\
(\contr 7) & \text{If }\phi^{\contr}_{\alpha\land \beta}\not\Pent \alpha\text{, then }\phi^{\contr}_{\alpha}\Pent\phi^{\contr}_{\alpha\land \beta} & \text{(conjunctive inclusion)}
\end{array}
\]
}


\nd Regarding the meaning of the above postulates, we redirect the reader to some of the many publications about the AGM approach, \eg~\cite{AlchourronEtAl1985,Caridroit17,FermeHansson2011}. The reader who is familiar with the AGM literature will also notice that we have reformulated the postulates considering formulae, as in~\cite{Caridroit17}, instead of logically closed theories. 
 
Now, we define the contraction operation associated with a KM $\km$ (\wrt~$\pr$) as follows. Consider formulae $\phi$ and $\alpha$. The set of \emph{possible remainders} of $\phi$ \wrt~$\alpha$, denoted $\phi\possremaind\alpha$, 
is defined as
%
\begin{equation}\label{possremind}
\phi\possremaind\alpha = 
\left\{
\begin{array}{ll}
\{\overl{([\phi]^+})\} & \text{if } \top \Pent \alpha\\
 \{\overl{([\psi]^+)} \mid \phi\Pent\psi \text{ and } \psi\not\Pent\alpha\} & \text{otherwise .}
\end{array} \right.
\end{equation}
%
\nd Essentially, as we would like to contract $\alpha$ from $\phi$ and, thus, would like to avoid that the contraction $\pr$-entails $\alpha$, we weaken $\phi$ so that we do not $\pr$-entail $\alpha$ anymore.
Please note that the set $\phi\possremaind\alpha$ is \emph{finite}, as for sets of formulae, we assumed that they cannot contain any pair of equivalent formulae (\cf~Remark~\ref{finitesets}).
Now, from the set of possible remainders, we select those that are most specific according to $\pr$-entailment, \ie
\begin{eqnarray}\label{remind}
    \phi\remaind\alpha & = & \{\psi \in \phi\possremaind\alpha \mid 
    \not\exists \psi' \in \phi\possremaind\alpha  \text{ s.t. } \psi'\PentL \psi \} .
\end{eqnarray}
\nd This set is called the set of \emph{remainders} of $\phi$ \wrt~$\alpha$.

Note that, if $\phi \not \Pent \alpha$ or $\top \Pent \alpha$ then the unique reminder is $\pr$-equivalent to $\phi$. On the other hand, it is not difficult to see that, if $\phi \Pent \alpha$, any remainder $\psi$ a formula whose models are those obtained by adding to the non-$0$ probability models of $\phi$ a single model of $\neg \alpha$. That is,

\begin{proposition} \label{compreminder}
Consider a probability distribution $\pr$ over $\W_\Sigma$ and formulae $\phi$ and $\alpha$ such that
$\phi \Pent \alpha$ and $\top \not\Pent \alpha$. Then
$\psi \in \phi\remaind\alpha$ iff 
%
\begin{equation*}\label{eq_contr}
\psi \equiv \overl({[\phi]_\Sigma^+ \cup \{w\}}) \ ,  
\end{equation*}
\nd for some $w \in [\neg \alpha]_\Sigma^+$. Furthermore, 
$\pr(\psi) = \pr(\overl{([\phi]_\Sigma^+})) + \pr(w)$. 
On the other hand, if $\phi \not \Pent \alpha$ or $\top \Pent \alpha$ holds then 
 $\phi\remaind\alpha = \{\overl{([\phi]^+})\}$.
\end{proposition}
\nd Eventually, contraction is defined in the following way.

%
\begin{definition}[KM-contraction]\label{def_KMcontraction}
Given formulae $\phi$, $\alpha$ and a KM $\km$, then 
the \emph{KM-contraction operator} $\contrk$ is defined as
%
\begin{equation}\label{eq_contraction_2}
\boxed{ \phi^{\contrk}_{\alpha}  =  
\bigvee  \km\min(\phi\remaind\alpha) 
} 
\end{equation}
%
\nd where, given a set of formulae $\Gamma$
\[
\km\min(\Gamma)=\{\phi\in\Gamma\mid \forall\psi\in\Gamma, \km(\phi) \leq \km(\psi)\} \ .
\]
\end{definition}

\nd The intuition behind the above definition is that we try to minimise the surprise among the most specific possible remainders.
But, let us go more in detail about the rationale behind this formalisation. 

The definition of appropriate procedures for modelling the belief dynamics of a \emph{rational} agent cannot be attained by relying only on purely logical constraints, since they are often insufficient to determine unequivocally which beliefs have to be dropped, as illustrated by the following example (taken from \cite[p.1]{FermeHansson2018}). A KB contains the statements ``Juan was born in Puerto Carre\~no'' ($\alpha$), ``Jos\'e was born in Puerto Ayacucho'' ($\beta$), and ``Two people are compatriots if they were born in the same country'' ($\gamma$). Assume that we are informed that ``Juan and Jos\'e are compatriots'' ($\delta$); if we simply add $\delta$ to the KB, we will obtain a contradiction,\footnote{Despite close to each other, Puerto Carre\~no is in Colombia, while Puerto Ayacucho is in Venezuela.} hence we have to modify our KB to accommodate the new information still preserving consistency. 
To this end, we have various options: for example, we could eliminate either $\alpha$, $\beta$, or even $\gamma$; or, among other options, we could weaken $\alpha$ and $\beta$ into, respectively, ``Juan was born in Puerto Carre\~no or Puerto Ayacucho'' and ``Jos\'e was born in Puerto Carre\~no or in Puerto Ayacucho''. 

To define procedures that model such changes, we may impose some extra-logical \emph{rationality criteria} that allow us to choose only one option among the different possible ones. Different criteria have been suggested that can be used to determine which KBs a rational agent should prefer if faced with different options (see \cite{RottPagnucco1999} for a more detailed presentation). 
One principle that is commonly accepted is the \emph{Principle of Indifference}, which is a basic principle connected to the formal management of preferences: 

\begin{equation}\label{pindiff}
\frac{\text{\emph{The\ Principle\ of\ Indifference}}}{\text{\emph{Objects held in equal regard should be treated equally}}} \ .
\end{equation}

\nd Another popular one is the \emph{Principle of Informational Economy} \cite{Gardenfors1988}, which simply states that we should not give up our beliefs beyond necessity: information is ``precious'', and if we need to change our beliefs, we should drop as little information as possible in the process. 
\begin{equation}\label{pinfoecon}
\frac{\text{\emph{The\ Principle\ of\ Informational Economy}}}{\text{\emph{Keep loss of information to a minimum }}} \ .
\end{equation}

\nd This is generally considered the guiding principle of the AGM approach (but it has also been argued that that is not the case \cite{Rott2000}). The Principle of Informational Economy is quite general and formally it can be interpreted in different ways. For example, it can be implemented in the form of the \emph{Principle of Conservatism} \cite{Harman1986}, where the relative informational difference between two KBs is defined in terms of set-theoretic inclusion: if a logical theory  $A$ contains the logical theory $B$, then $A$ is considered more informative than $B$ and it is preferred or, equivalently, if a formula $\phi$ implies a formula $\psi$, $\phi$ is considered more informative than $\psi$ (compare also with Proposition~\ref{prKM}). 
Adhering to such a principle together with the Principle of Indifference results in the so-called \emph{full-meet contraction}, that is, to contract $\phi$ by $\alpha$ we take under consideration \emph{only} the formulas in the set of remainders (according to the Principle of Conservatism), and \emph{all} of them (according to the Principle of Indifference). That is, the full-meet contraction is the following operation

\begin{equation}\label{eq_contraction_fullmeet}
\phi^{\contr}_{\alpha}  =  
\bigvee  (\phi\remaind\alpha).
\end{equation}

\nd It is well-known \cite[Sect. 3.8]{FermeHansson2018} that full-meet contraction is not a satisfying approach to contraction, since it tends to drop too much information. 
A better behaviour is obtained using a \emph{partial-meet contraction}, that is, considering only \emph{some} of the elements of the set $\phi\remaind\alpha$, \ie
\begin{equation}\label{eq_contraction_partialmeet}
\phi^{\contr}_{\alpha}  =  
\bigvee  \gamma(\phi\remaind\alpha),
\end{equation}
\nd where $\gamma$ is a  so-called \emph{choice function}.\footnote{$\gamma$ returns a set of formulae.} Different implementations of the partial-meet contraction depend on the rationale that we consider in the definition of the choice function. 
%
In this work, we propose to use KMs as a way of modelling the expectations of a rational agent. To this end, we introduce the 
\emph{Principle of Minimal Surprise}, where ``surprise" is expressed in information-theoretic terms, \ie~the more probable an event is the less it is surprising.
\begin{equation}\label{pminsur}
\frac{\text{\emph{The\ Principle\ of\ Minimal\ Surprise}}}{\text{\emph{The\ less\ surprising\ option\ should\ be\ preferred}}} \ . 
\end{equation}

\nd This principle simply states that facing uncertainty about which is the actual world, an agent should opt for the most expected (that is, the least surprising) option. Such a principle is intuitive, and it has been also assumed in major approaches to conditional non-monotonic reasoning, such as~\cite{GardenforsMakinson1994,Spohn2012}.
%
%
Our approach follows the same line, but it uses KMs in order to create a bridge between our probabilistic framework and the ranking of the information in terms of level of surprise.
%
%
Note that in the KM theory, the informative value associated with a piece of information is directly proportional to its exceptionality: the less probable it is, the more informative it conveys (it is an immediate consequence of the axiom {\bf (KM2)}). Hence, a formula $\phi$ is strictly less surprising than a formula $\psi$ iff $\kms(\phi)<\kms(\psi)$.


Let us now go back to Definition~\ref{def_KMcontraction}. We want to contract $\phi$ by $\alpha$. To do that, we need to select some formulas in the remainder set $(\phi\remaind\alpha)$, that, relying on KMs and the Principle of Minimal Surprise, results in choosing the most expected ones. The result is the equation expressed in (\ref{eq_contraction_2}).

By Proposition \ref{compreminder} we see that this operation corresponds, from the point of view of the semantics, to selecting the most expected worlds in $[\neg \alpha]_\Sigma^+$. Moreover, if we want to be compliant also with the Principle of Indifference, we will have to consider \emph{all} the most expected worlds in $[\neg\alpha]_\Sigma^{+}$, as formalised below in Proposition \ref{compreminderB}.

Note that $\phi^{\contrk}_{\alpha}$ is deterministic in the sense that it is uniquely determined by the given probability distribution (and different probability distributions may give rise to different contractions).
%


Moreover, given a contraction operator $\contr$, we may also associate with it a quantitative function that numerically specifies how much information has been lost by contracting $\phi$ with $\alpha$. That is, the \emph{information loss of a contraction $\contr$} (\wrt~a given probability distribution $\pr$) is defined as
\begin{equation} \label{losscontra}
L^\contr_\pr(\phi, \alpha) = \km(\phi) - \km(\phi^{\contr}_{\alpha}) \ .
\end{equation}

\nd Note that $L^\contr_\pr(\phi, \alpha)\geq 0$.

\begin{remark}
    Please note that, for any AGM contraction operator $\contr$ on classical propositional logic, we may use the definition of information loss of a contraction by reverting equivalently to Straccia's formulation of KM (see Eq.~\ref{kmhs}) and, thus, in this case the loss can be quantified as 
    $L^\contr(\phi, \alpha) = \kmh(\phi) - \kmh(\phi^{\contr}_{\alpha})$.
\end{remark}



\nd Proposition~\ref{compreminder} provides us with a simple constructive way to compute a contraction. That is, given $\pr$, $\phi$ and $\alpha$, $\phi^{\contrk}_\alpha$ can be computed in the following way:
\begin{enumerate}
    \item if $[\phi]_\Sigma^{+} \not \subseteq [\alpha]_\Sigma$ or $[\top]_\Sigma^{+} \subseteq [\alpha]_\Sigma$ then return $\overl{([\phi]_\Sigma^+})$ and we are done; otherwise
    \item determine the remainder set $\phi \remaind \alpha$ as the set of formulae $\psi_i$, where for each $w_i \in [\neg \alpha]_\Sigma^{+}$, $\psi_i = \overl({[\phi]_\Sigma^+ \cup \{w_i\}})$; and
    \item eventually, compute $\km(\psi_i)$ from $\pr(\psi_i)$, determine so $\km \min(\phi \remaind \alpha )$ and, conclude by applying Eq.~\ref{eq_contraction_2}.
\end{enumerate}


\begin{example}[Running example cont.] \label{runexG}
  Assume that we would like to contract $\KB$ with $f$, \ie~let us determine 
  $\KB \contr_{\kms} f$. We already know from Example~\ref{runexD} that
  $\KB \PentPS f$ and, thus, $\KB$ can not be itself a possible reminder. Therefore, we need to weaken it. 
To do so,   at first, we build $\pr_\Sigma0$, according to Example~\ref{runexC}, and then determine the remainder set by using Proposition~\ref{compreminder}  and the models of $\KB$ identified in Example~\ref{runexA}  together with their probabilities as from Example~\ref{runexC}. 
  

Now, the non-zero probability models of $\neg f$ are those in $\pr_\Sigma$ for which $f$  is false. These can be identified by: 
\begin{eqnarray*}
    [\neg f]^{+}_\Sigma &= & \{ bp\bar{o}\bar{f}w^{0.15},  b\bar{p}o\bar{f}w^{0.2}, b\bar{p}\bar{o}\bar{f}w^{0.2}\} \ .
\end{eqnarray*} 

\nd As a consequence, we get three remainders, one for each element in $[\neg f]^{+}_\Sigma$. That is,
\begin{eqnarray*}
    \psi_1 & = & \overl({[\KB]_\Sigma \cup \{bp\bar{o}\bar{f}w\}})   \\
    \psi_2 & = & \overl({[\KB]_\Sigma \cup \{b\bar{p}o\bar{f}w\}})  \\
    \psi_3 & = &  \overl({[\KB]_\Sigma \cup \{b\bar{p}\bar{o}\bar{f}w\}}) \ .
\end{eqnarray*} 

\nd Moreover, it is easily verified that
  \begin{eqnarray*}
\pr_\Sigma(\psi_1) & = & 0.35+0.15= 0.5  \\
\pr_\Sigma(\psi_2) & = & 0.35+0.2= 0.55 \\
\pr_\Sigma(\psi_3) & = & 0.35+0.2= 0.55 
  \end{eqnarray*}

\nd and consequently, 
$\kms(\psi_1)  =  1.0$,
$\kms(\psi_2)  \approx  0.862$, and
$\kms(\psi_3)  \approx  0.862$. 
Eventually, the knowledge minimal (least surprising/most probable) ones are $\psi_2$ and $\psi_3$ and, thus, we conclude with
\begin{eqnarray}
\KB \contr_{\kms} f & \equiv & \psi_1 \lor \psi_2  \nonumber \\
& \equiv & (\KB \lor \overl{b\bar{p}o\bar{f}w} \lor \overl{b\bar{p}\bar{o}\bar{f}w}) \nonumber \\ 
& \equiv & (\KB \lor \overl{b\bar{p}\bar{f}w}) \ . \label{contrkf} 
\end{eqnarray}

\nd Therefore, the contraction weakens the initial belief in Example~\ref{runexA} by making the belief `birds not being penguins do not fly and have wings' possible as well (\wrt~the given probability distribution). The latter is also the `least surprising' (most probable) candidate among the remainders. Moreover,  note that $\pr_\Sigma(\KB \contr_{\kms} f) = 0.75$, so 
$\kms(\KB \contr_{\kms} f) \approx 0.415$, and, thus, the information loss of this contraction can be quantified as 
\begin{equation} \label{losskf}
L^{\contr_{\kms}}_{\pr_\Sigma}(\KB,f) = 1.515 - 0.415 = 1.1 \ .
\end{equation}
\end{example}

\nd From Proposition~\ref{compreminder}, and, from what we have seen in Example~\ref{runexG} above, it is then not surprising that the following proposition holds.

\begin{proposition} \label{compreminderB}
Consider a probability distribution $\pr$ over $\W$, a KM $\km$ and formulae $\phi$ and $\alpha$. 
Then
%

\begin{equation} \label{eq_contrA}
\phi^{\contrk}_{\alpha} \equiv
\left \{ \begin{array}{ll}
\overl({[\phi]^+ \cup \min_{\km}([\neg \alpha]^+})) & \text{if } \phi\Pent \alpha\\
\overl{[\phi]^+} & \text{otherwise, }
\end{array}\right .
\end{equation}

\nd where $\min_{\km}([\neg \alpha]^+)=\{w\in\W \mid w\in[\neg \alpha]^+\text{ and there is no } w'\in\W\text{ s.t. } 
w'\in[\neg \alpha]^+ \text{ and } \km(w') < \km(w)\}$.
\end{proposition}

\nd A consequence of Proposition~\ref{compreminderB} is the following proposition establishing exactly the information loss of a contraction of $\phi$ \wrt~$\alpha$  in terms of the probability of $\phi$ and the probability of having a knowledge minimal (least surprising) $\neg \alpha$ world.

\begin{corollary} \label{compreminderBB}
Consider a probability distribution $\pr$ over $\W$, a KM $\km$ and formulae $\phi$ and $\alpha$. 
Then
\begin{equation} \label{eq_contrAA}
L^\contr_\pr(\phi, \alpha) = 
\left \{ \begin{array}{ll}
\log_2 (1+ \frac{\pr(\overl\min_{\km}([\neg \alpha]^+)}{\pr(\phi)}) & \text{if } \phi\Pent \alpha\\
0 & \text{otherwise . }
\end{array}\right .
\end{equation}
\end{corollary}




\nd So, for instance, by referring to Example~\ref{runexG} and Eq.~\ref{contrkf}, we have $\pr(\KB) = 0.35$ and $\pr(\overl{b\bar{p}\bar{f}w})) = \pr(\overl{b\bar{p}o\bar{f}w})) + \pr(\overl{b\bar{p}\bar{o}\bar{f}w})) = 0.2 + 0.2 = 0.4$ and, thus,
\begin{eqnarray*}
 L^\contr_\pr(\KB, f)   & = & \log_2 (1+ \frac{0.4}{0.35}) \\
 & = & 1.1 \ ,
\end{eqnarray*}

\nd which confirms Eq.~\ref{losskf}.

The following proposition tells us that a KM-contraction operator\footnote{We recall that different probability distributions may give rise to different KM-contraction operators.} is an AGM contraction operation.

\begin{proposition}\label{prop_AGM_sound}
Consider a probability distribution $\pr$ and a KM-contraction operator $\contrk$ based on $\pr$. Then $\contrk$ is an AGM contraction operator, that is, it satisfies postulates $(\contr 1)$-$(\contr 7)$.
\end{proposition}

\nd Next, we prove that every AGM contraction operation can be represented as a KM-contraction operation by appropriately defining the probability distribution underlying the KM-contraction operation. That is, the various AGM contraction operations differentiate each other on a probability distribution over worlds they rely on.\footnote{Of course, such a characterisation would not be possible using Straccia's $\kmh$, where the uniform distribution is assumed.} 

To start with, it is well-known that AGM operations can be characterised by a total preorder over the set of interpretations $\W_\Sigma$.  Formally, given a formula $\phi$, a total preorder $\leq_{\phi}$ on $\W_\Sigma$, 
(we denote the related strict relation as $<_{\phi}$, while denote related symmetric relation as $\simeq_{\phi}$),
is  a {\em  faithful assignment for $\phi$} iff the following conditions hold: 

\begin{enumerate}
    \item if $\{w, w'\}  \subseteq [\phi]_\Sigma$, then $w \simeq_{\phi} w'$; and
    \item if ${w} \in [\phi]_\Sigma$ and ${w'} \not \in [\phi]_\Sigma$, then $w <_{\phi} w'$.
\end{enumerate}

\nd Eventually, given a preorder $\leq$ over worlds,
\[
\min_{\leq}([\psi]_\Sigma) = \{w\in [\psi]_\Sigma\mid w\leq w', \text{ for every } w'\in[\psi]_\Sigma\} \ .
\]

\nd The notion of a faithful assignment allows us to characterise contraction and revision operations satisfying $(\contr 1)-(\contr 7)$. The following is a representation result that can be easily derived and is the KM-based analogue to~\cite[Theorem 14]{Caridroit17} and~\cite[Theorem 3.3]{KatsunoMendelzon1991}.\footnote{The reader may compare it also with Proposition~\ref{compreminderB}.}


%
\begin{proposition}[\cite{Caridroit17,KatsunoMendelzon1991}]\label{KMranking}
Let $\phi,\alpha$ be formulae.
An operation $\contr$ on ${\phi}$ satisfies $(\contr 1)-(\contr 7)$ if and only if there is  a  faithful assignment $\leq_{\phi}$ for $\phi$ such that 
\[
{\phi}^{\contr}_{\alpha} \equiv \overl({[\phi] \cup    \min_{\leq_{\phi}}([\neg \alpha]})) \ .
\]
\end{proposition}

\nd So, this proposition tells us that we may express the set of models of the contraction of $\phi$ by $\alpha$ as the union of the models of $\phi$ and of the minimal counter-models of $\alpha$ \wrt~$\leq_\phi$.

Given that $\Sigma$ is finite, the total preorder $\leq_{\phi}$ can be easily translated into a ranking function $r_{\phi}$: namely
%
\begin{equation}
  r_{\phi}(w)=\begin{cases}
    0, & \text{if }w\in\min_{\leq_{\phi}}(\W_\Sigma)\\
    i, & \text{if }w\in\min_{\leq_{\phi}}(\W_\Sigma\setminus \{w'\mid r_{\phi}(w')<i\}) \ ,
  \end{cases}
\end{equation}
\nd which we call a \emph{$\phi$-faithful ranking}. Note that models of $\phi$ have rank $0$.
The contraction $\contr$ associated to $r_{\phi}$  is then defined as
%
\begin{equation}\label{eq_ranked_contr}
{\phi}^{\contr}_{\alpha} \equiv \overl({[\phi] \cup    \min_{r_\phi}([\neg \alpha]})) \ .
\end{equation}
%
\begin{example}\label{ex_contr_1}
Let $\Sigma=\{a,b,c\}$ and $\phi\defined a\land b$. A $\phi$-faithful ranking $r$ is shown in Table \ref{Fig_1}. The other columns will be addressed in Example~\ref{ex_contr_2} later on.
To determine $(a\land b)^{\contr_{r}}_{b}$, we have to consider $\min_{r}([\neg b])=\{a\bar{b}c,a\bar{b}\bar{c}\}$. Hence $[(a\land b)^{\contr_r}_{b}]=\{abc,ab\bar{c},a\bar{b}c,a\bar{b}\bar{c}\}$ and, thus, $(a\land b)^{\contr_r}_{b} \equiv a$.
\begin{table}[h]
\caption{$\phi$-faithful ranking $r$ with possible $r$-faithful probability distribution $\pr$, and the relative KM.} \label{Fig_1}
\begin{center}
{
\begin{tabular}{|c|c|c|c|} \hline
$r$ & worlds & $\pr$ & $\kms$ \\ \hline
{$2$} & $\bar{a}\overline{b}{c}$ \hspace{0.2cm} $\bar{a}\overline{b}\bar{c}$ & $1/16$ & $4$ \\ \hline
{$1$} & $a\overline{b}c$ \hspace{0.2cm} $a\overline{b}\bar{c}$ \hspace{0.2cm} $\bar{a}bc$ \hspace{0.2cm} $\bar{a}{b}\bar{c}$ & $2/16$ & $3$\\ \hline
{$0$} & $abc$ \hspace{0.2cm} $ab\bar{c}$ & $3/16$ & $2.415$ \\ \hline
\end{tabular}
}
\end{center}

\end{table}
\end{example}

\nd Now, given any set of worlds $\W$ and any $\phi$-faithful ranking $r$ over $\W$, we can define an $r$-\emph{faithful probability distribution} $\pr_{r_\phi}$ over $\W$ respecting the following constraints:
\ii{i} for any $w,w'\in\W$, $\pr_{r_\phi}(w)\geq\pr_{r_\phi}(w')$ iff $r_\phi(w)\leq r_\phi(w')$; and
\ii{ii} for any $w\in\W$, $\pr_{r_\phi}>0$.

\begin{proposition}\label{prop_faith_prob}
For any formula $\phi$, set of worlds $\W$, and $\phi$-faithful ranking $r$ over $\W$, it is always possible to define an $r$-faithful probability distribution $\pr$ over $\W$.
\end{proposition}
%
%
%
%
%
\nd Now, using $\pr$ constructed as in Proposition~\ref{prop_faith_prob}, we can generate a KM-contraction that corresponds to the AGM-contraction we started with.

\begin{proposition}\label{prop_AGM_compl}
Let $\phi$ be any formula and $\contr$ any AGM-contraction for $\phi$. We can define a KM contraction $\contr_{\km}$ s.t., for all $\alpha$,
$\phi^\contr_\alpha \equiv \phi^{\contr_{\km}}_\alpha$ holds.
\end{proposition}
\begin{example}\label{ex_contr_2}
Consider Table \ref{Fig_1}. The $\pr$ (resp.~$\kms$) column indicates world's probabilities (resp.~KM) within the same rank according to an $r$-faithful probability distribution $\pr$, as by  Proposition~\ref{prop_faith_prob}. For instance, if $\pr(abc)=3/16$ (resp.~$\kms(abc)=2.415$). 
%
To determine $(a\land b)^{\contr_{\kms}}_{b}$ \wrt~that $\pr$, we have to consider $\min_{\kms}([\neg b])=\{a\bar{b}c,a\bar{b}\bar{c}\}$. Hence, $[(a\land b)^{\contr_{\kms}}_{b}]=\{abc,ab\bar{c},a\bar{b}c,a\bar{b}\bar{c}\}$ and $(a\land b)^{\contr_{\kms}}_{b}\equiv a \equiv (a\land b)^{\contr_{r}}_{b}$, in agreement with Proposition~\ref{prop_AGM_compl}.
\end{example}
\nd A consequence of Propositions \ref{prop_AGM_sound} and \ref{prop_AGM_compl} is the following.
\begin{proposition}\label{th_repr_contr}
A function $\contr$ is an AGM contraction operation iff it is a KM-contraction operation.
\end{proposition}

\subsection{Sphere-based view of contraction} \label{sbs}

\nd An interesting way of viewing a contraction operator may be given pictorially in terms of a \emph{system of spheres} (see, \eg~\cite{Rott99,Grove88}), as illustrated in Figure~\ref{figsphere1} (a). 
\begin{figure}[!t]
\centering 
\begin{tabular}{cc}
\includegraphics[scale=0.30]{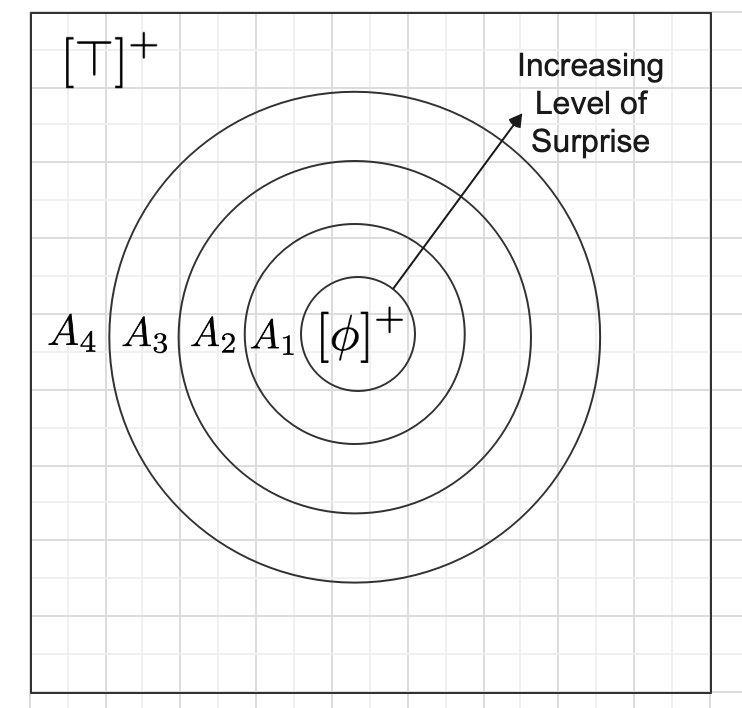} & 
\includegraphics[scale=0.30]{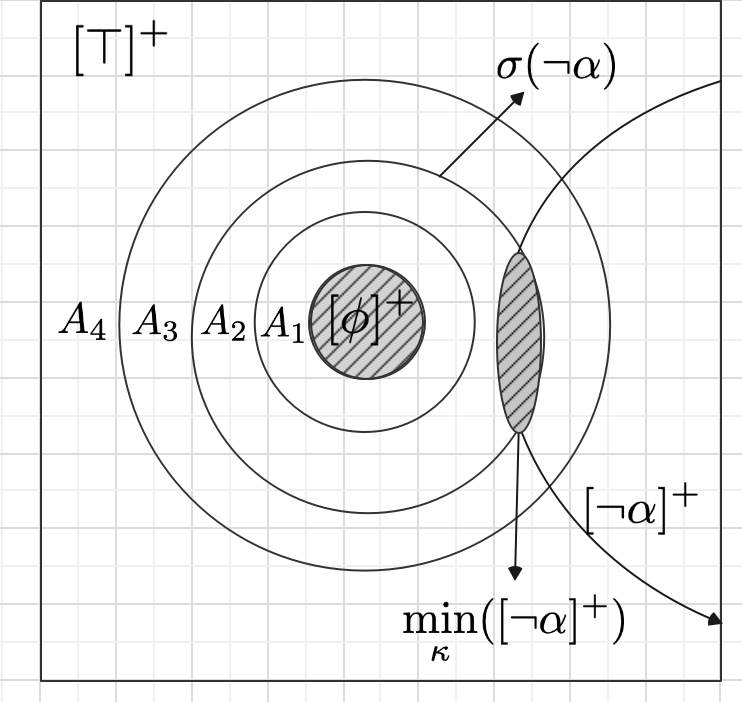} \\
(a) & (b)  \\
\end{tabular}
\caption{(a) System of spheres; (b) System of spheres with contraction $\phi^\contr_\alpha$ shaded.}
\label{figsphere1}
\end{figure}
Specifically,  consider a formula $\phi$ and a formula $\alpha$ to be contracted from $\phi$.
A \emph{sphere} $\sigma$ is a set of possible worlds.\footnote{That is, worlds having non-zero probability.}
A \emph{system of spheres} ${\ss}_\phi$ centred on $\phi$~\footnote{We will omit the reference to  $\ss_\phi$ if clear from context.} is an order $\sigma_0 < \sigma_1 \ldots < \sigma_n$ of nested spheres
(in the sense of set inclusion $\sigma_i \subset \sigma_{i+1}$, $0 \leq i < n$) in which the smallest or innermost sphere is $\sigma_0=[\phi]^+$ (the possible models of $\phi$), the outermost is $\sigma_n = [\top]^+$ (all possible models) and where
\begin{enumerate}
    \item all possible worlds in the \emph{annulus} $A_{i+1} = \sigma_{i+1} \setminus \sigma_i$ ($0\leq i < n$) have the same probability, \ie~for $w,w'\in A_{i+1}, \pr(w) = \pr(w')$; and
    \item the probability of a world decreases from the innermost annulus $A_1$ to the outermost annulus $A_{n}$, 
    \ie~for $w\in A_i$ and $w' \in A_{i+1}$ ($i > 0$), we have $\pr(w) > \pr(w')$.
\end{enumerate}
\nd Intuitively, the agent believes the actual world to be one
of the $\phi$-worlds but does not have sufficient information to establish which
one. However, the agent may be mistaken, in which case it believes that the actual world is most likely to be one of those in the next greater sphere and so on. That is, with 2. we want to say that the nesting of the spheres defines an order in terms of \emph{increasing surprise} of the worlds an agent considers possible. Specifically, moving from the innermost annulus to the outermost one increases the surprise, from an information-theoretic point of view. On the other hand, point 1, encodes the fact that each possible world within an annulus is equally plausible and there is no reason to prefer one world to the other if we accept the Principle of Indifference. 
Now, in order to contract $\alpha$ from $\phi$~\footnote{For ease, assume $[\phi]^+ \subseteq [\alpha]^+$ and $[\neg \alpha]^+ \neq \emptyset$.} we need to add at least one $\neg \alpha$ model to the models of $\phi$. As we are inspired by the Principles of Minimal Surprise and Indifference, to do so we consider the smallest sphere $\sigma(\neg \alpha))$ intersecting $[\neg \alpha]^+$, which is indeed $\min_{\km}([\neg \alpha]^+) = \sigma(\neg \alpha) \cap [\neg \alpha]^+$ (see Figure~\ref{figsphere1} (b)). By the Principle of Indifference, each world in this intersection is equally plausible and, thus, the possible models of $\phi^\contr_\alpha$ (see also by Proposition~\ref{compreminderB}) are those in $[\phi]^+ \cup \min_{\km}([\neg \alpha]^+)$, \ie~the shaded ones in Figure~\ref{figsphere1} (b).

\begin{remark}\label{remsigma}
Note that we may express $\sigma(\neg \alpha)$ in the following way. 
For a formula $\beta$, let $p_{\beta}^{\max}$ be the maximal probability of the possible $\beta$-models, \ie 
\[
p_{\beta}^{\max} = \max \{ \pr(w) \mid w \in [\beta]^+ \} \ .
\]
\nd Note that $\pr(w) = p_{\beta}^{\max}$ for all $w\in\min_\km [\beta]^+$. Then 
\begin{equation} \label{sigmaalpha}
    \sigma(\neg \alpha)) = [\phi]^+ \cup \{w \in [\top]^+ \setminus [\phi]^+ \mid \pr(w) \geq p_{\neg \alpha}^{\max} \} \ .
\end{equation}
\end{remark}

\subsection{A contraction without recovery} \label{contrarecovery}
\nd As we have seen, our contraction operator is an AGM contraction and, thus, satisfies the (in)famous \emph{recovery} postulate $(\contr 5)$ around which there is a longstanding debate. Various contraction proposals have been developed not satisfying $(\contr 5)$, such as \emph{saturatable contraction}~\cite{Levi91},  \emph{semi contraction}~\cite{ferme98}, \emph{severe withdrawal}~\cite{Rott99}, \emph{recuperative/ring withdrawal}~\cite{Ferme24}, and  more (see \eg~\cite{Ferme18}). 

The sole purpose of this section is to illustrate how one may build from our KM-based contraction operator  one not satisfying the recovery postulate. To do so, we focus on \emph{severe withdrawal}~\cite{Rott99}, denoted $\ddiv$, since one may give for it a simple sphere-based pictorial explanation and formalisation.\footnote{For many other approaches and their sphere-based view, see \eg~\cite[Appendix B]{Rott99}.}

From a sphere-based point of view, given a formula $\phi$ and a formula $\alpha$ to be contracted,\footnote{Again, for ease we assume $ [\neg \alpha]^+ \neq \emptyset$.} the severe withdrawal of $\alpha$ from $\phi$, denoted $\phi^{\ddiv}_\alpha$, is the shaded area in Figure~\ref{figspheresevere} (compare with Figure~\ref{figsphere1} (b)).
\begin{figure}[!t]
\centering 
\begin{tabular}{c}
\includegraphics[scale=0.30]{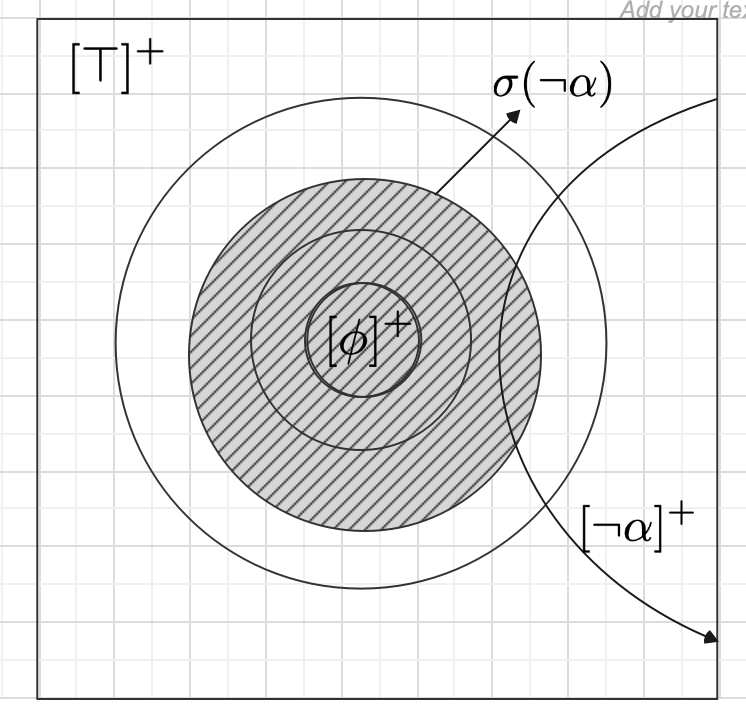} 
\end{tabular}
\caption{System of spheres for severe withdrawal showing $\phi^{\contrs}_\alpha$ shaded.}
\label{figspheresevere}
\end{figure}
Essentially, severe withdrawal relies also on the \emph{Principle of Weak Preference}, according to which ``if one world is considered at least as plausible as another then the former should be considered if the latter is. That is, an agent has determined a preference over worlds (the spheres) and does not prefer the (closest) $\neg \alpha$-worlds over the (closer) $\alpha$-worlds just because it is giving up belief in $\alpha$. 
Its preferences are established prior to the change and we assume that there is no reason to alter them in light of the new information"~\cite{Rott99}. Consequently, by following~\cite{Rott99}, we may define our km-based \emph{severe withdrawal} operator as the formula corresponding to the set of possible worlds in the smallest sphere intersecting $[\neg \alpha]^+$, 
\ie~$\KB^{\ddiv}_\alpha =  \overl{\sigma(\neg \alpha)}$.


Formally, as we did for contraction, we provide next the $\pr$-analogue postulates of sever withdrawal given in~\cite{Rott99} as follows:
a function $\contrs$ is a \emph{severe withdrawal} function if it satisfies the following postulates:

\[
\begin{array}{lll}
(\contrs 1)    &  \phi\Pent \phi^{\contrs}_{\alpha} & \\
(\contrs 2)    & \text{If } \phi\not\Pent \alpha \text{ or } \top \Pent \alpha \text{, then } \phi^{\contrs}_{\alpha}\equiv_{\pr} \phi & \\
(\contrs 3) & \text{If } \top \not\Pent \alpha\text{, then }\phi^{\contrs}_{\alpha}\not\Pent \alpha & \\ 
(\contrs 4) & \text{If }\alpha\equiv_{\pr} \beta\text{, then }\phi^{\contr}_{\alpha}\equiv_{\pr} \phi^{\contr}_{\beta} & \\
(\contrs 6a) & \text{If } \top \not\Pent \alpha \text{, then } \phi^{\contrs}_{\alpha\land \beta}\Pent\phi^{\contrs}_{\alpha} & \\ 
(\contrs 7) & \text{If }\phi^{\contrs}_{\alpha\land \beta}\not\Pent \alpha\text{, then }\phi^{\contrs}_{\alpha}\Pent\phi^{\contrs}_{\alpha\land \beta} & 
\end{array}
\]
\nd Postulates $(\contrs 1)$, $(\contrs 3)$, $(\contrs 4)$ and $(\contrs 7)$ are as for AGM contraction. Postulate $(\contrs 2)$, is as the vacuity postulate $(\contr 3)$, but contains an additional antecedent $\top \not\Pent \alpha$ to take care of the limiting case, which was previously handled with the aid of the recovery postulate, which now is missing. Postulate $(\contrs 6)$ has been replaced with the stronger \emph{antitony condition} $(\contrs 6a)$, which states that anything that is given up to remove a strong sentence (the conjunction of $\alpha$ and $\beta$) should also be given up when removing a weaker sentence ($\alpha$), provided the latter is not logically true~\cite{Rott99}. 

We now show, by relying on $\pr$-analogue construction as proposed by~\cite{Rott99}, how to define a severe withdrawal function $\contrks$ from our KM-based contraction function $\contrk$.

\begin{definition}[KM-Severe withdrawal from $\contr$]\label{def_severecontr}
Given formulae $\phi$, $\alpha$ and a KM-contraction $\contrk$. Then the \emph{KM-severe  withdrawal function defined from $\contrk$} is defined as follows: 
\begin{equation} \label{eq_contrASevere}
\phi^{\contrks}_{\alpha} = 
\left \{ \begin{array}{ll}
\bigwedge \{ \beta \mid \phi^{\contrk}_{\alpha \land \beta} \Pent \beta  \} & 
\text{if } \top \not \Pent \alpha \\
\overl{[\phi]^+} & \text{otherwise.}
\end{array}\right .
\end{equation}


\end{definition}

\nd \nd Please note that again the set $\{ \beta \mid \phi^{\contrk}_{\alpha \land \beta} \Pent \beta  \}$ is finite, as for sets of formulae, we assumed that they cannot contain any pair of equivalent formulae (\cf~Remark~\ref{finitesets}).

Now, observe that in order $\phi^{\contrk}_{\alpha \land \beta} \Pent \beta$ to hold, it should be the case that the smallest sphere $\sigma(\neg \alpha)$ intersecting $[\neg \alpha]^+$ has to be interior to the smallest 
sphere $\sigma(\neg \beta)$ intersecting $[\neg \beta]^+$, \ie~$\sigma(\neg \alpha) < \sigma(\neg \beta)$, as otherwise we would add $\neg \beta$-worlds to $\phi^{\contrk}_{\alpha \land \beta}$ and, thus, $\phi^{\contrk}_{\alpha \land \beta} \Pent \beta$ would not hold. Consequently, $\sigma(\neg \alpha)$ contains no $\neg \beta$-worlds, \ie~$\beta$ is satisfied by all worlds in  $\sigma(\neg \alpha)$ (see Figure~\ref{figspheresevereB}).
\begin{figure}[!t]
\centering 
\begin{tabular}{c}
\includegraphics[scale=0.30]{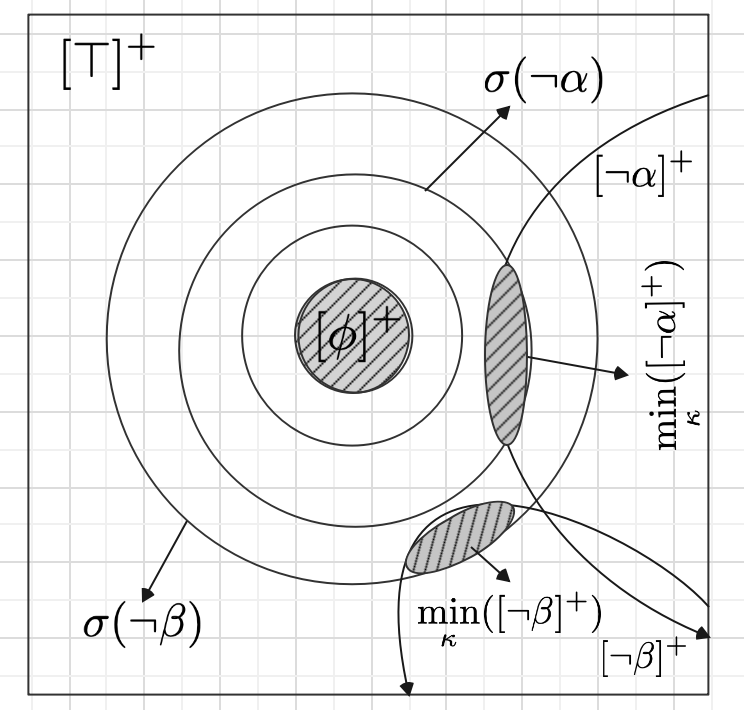} 
\end{tabular}
\caption{System of spheres where $\phi^{\contrk}_{\alpha \land \beta} \Pent \beta$ holds. }
\label{figspheresevereB}
\end{figure}
Therefore, $\{ \beta \mid \phi^{\contrk}_{\alpha \land \beta} \Pent \beta  \}$ is the set of formulae that hold in $\sigma(\neg \alpha)$, which in turns mean that $\phi^{\ddiv}_\alpha =  \overl{\sigma(\neg \alpha)}$. From what we have said, the following Proposition is immediate.

\begin{proposition}\label{propseveresigma}
Given formulae $\phi$, $\alpha$, a KM-contraction $\contrk$ and the KM-severe  withdrawal function $\contrks$  defined from $\contrk$.
Then 
\begin{equation} \label{eq_contrASeveresigma}
\phi^{\contrks}_{\alpha} \equiv \left \{ \begin{array}{ll}
\overl{\sigma(\neg \alpha)}  & \text{if }\top \not \Pent \alpha \\
\overl{[\phi]^+} & \text{otherwise}.
\end{array}\right .
\end{equation}
\nd where $\sigma(\neg \alpha)$ is defined as in Eq.~\ref{sigmaalpha}.
\end{proposition}

\nd In other words, ``by contracting $\alpha$ from $\phi$ according to Definition~\ref{def_severecontr}, we keep those formulae $\beta$ that would be retained when given the choice to remove either $\alpha$ or $\beta$ (or both)"~\cite{Rott99}.


\begin{remark} \label{fermring}
In~\cite{Ferme24} it is shown that one may define a ring withdrawal operator using a severe withdrawal operator
or even using an AGM contraction operator. Along a similar line, following what we have done in this section, it not difficult to see that one may define the analogue KM-based ring withdrawal operator using a KM-based severe withdrawal operator or even using a KM-based AGM contraction operator.
\end{remark}

\nd By Proposition~\ref{propseveresigma}, obviously 
\begin{equation} \label{moresevere}
    \phi^{\contrk}_{\alpha} \Pent \phi^{\contrks}_{\alpha} \, 
\end{equation}

\nd holds (see also Figure~\ref{figspheresevere}). Consequently, both
\[
\km(\phi^{\contrk}_{\alpha}) \geq \km(\phi^{\contrks}_{\alpha})
\]
\nd and
\[
L^{\contrk}_\pr(\phi, \alpha) \leq L^{\contrks}_\pr(\phi, \alpha) 
\]

\nd hold, confirming that severe withdrawal loses more knowledge/information than an AGM contraction. Specifically, the loss of the KM-based severe withdrawal function and analogue of Corollary~\ref{compreminderBB} is 

\begin{corollary} \label{compreminderBBS}
Consider a probability distribution $\pr$ over $\W$, formulae $\phi$ and $\alpha$. 
Then
\begin{equation} \label{eq_contrAAS}
L^{\contrks}_\pr(\phi, \alpha) = 
\left \{ \begin{array}{ll}
\log_2 \frac{\pr(\overl\sigma(\neg \alpha))}{\pr(\phi)} & 
\text{if }  \top \not \Pent \alpha \\
0 & \text{otherwise . }
\end{array}\right .
\end{equation}
\end{corollary}

\nd Note that, if $\top \not \Pent \alpha$, as $[\phi]^+ \subseteq \sigma(\neg \alpha)$, we have $\frac{\pr(\overl\sigma(\neg \alpha)}{\pr(\phi)} \geq 1$. So, we may also write $L^{\contrks}_\pr(\phi, \alpha)$ as
\begin{equation} \label{eq_contrAAbis}
L^{\contrks}_\pr(\phi, \alpha) = 
\left \{ \begin{array}{ll}
\log_2 (1 + \frac{\pr(\overl (\sigma(\neg \alpha) \setminus [\phi]^+) )}{\pr(\phi)} & 
\text{if } \top \not \Pent \alpha \\
0 & \text{otherwise . }
\end{array}\right .
\end{equation}

\nd We conclude this section by showing that $\contrks$ is indeed a severe withdrawal function in the sense that it satisfies the postulates for severe withdrawal.

\begin{proposition}\label{prop_SW_sound}
Consider a probability distribution $\pr$ and a KM-contraction operator $\contrk$ based on $\pr$. Then the function $\contrks$ is a severe withdrawal operator, that is, it satisfies the postulates $(\contrs 1)$-$(\contrs 4)$, $(\contrs 6a)$ and $(\contrs 7)$.
\end{proposition}

\nd We can also prove the opposite direction, that is,  for each severe withdrawal operator $\contrs$ there is a KM $\km$ s.t. $\contrks$ corresponds to $\contrs$.

\begin{proposition}\label{th_repr_contrs}
A function $\contrs$ is a severe withdrawal operator iff it is a KM-severe withdrawal operation. 
\end{proposition}




\subsection{Belief revision} \label{belrev}

\nd Now, we address the other main BC operation, namely \emph{revision}. A function $\star$ is an AGM \emph{revision} in our setting iff it satisfies the following postulates, which are the $\pr$-entailment analogues of those in~\cite{Caridroit17}):\footnote{Note that $(\star 2)$ implies $(\star 1)$.}
{
\[
\begin{array}{lll}
(\star 1)    &  \phi^{+}_{\alpha}\Pent \phi^{\star}_{\alpha} & \text{(inclusion)}\\
(\star 2)    & \text{If }\phi\not\Pent \neg \alpha\text{, then }\phi^{\star}_{\alpha}\equiv_{\pr} \phi^{+}_{\alpha} & \text{(vacuity)}\\
(\star 3) & \phi^{\star}_{\alpha}\Pent \alpha & \text{(success)}\\
(\star 4) & \text{If }\alpha\equiv_{\pr} \beta\text{, then }\phi^{\star}_{\alpha}\equiv_{\pr} \phi^{\star}_{\beta} & \text{(extensionality)} \\
(\star 5) & \text{If }\alpha\not\Pent\remaind\text{, then }\phi^{\star}_{\alpha}\not\Pent\remaind & \text{(consistency)} \\
(\star 6) &(\phi^{\star}_{\alpha})^{+}_{\beta} \Pent \phi^{\star}_{\alpha\land \beta}& \text{(superexpansion)}\\
(\star 7) & \text{If }\phi^{\star}_{\alpha}\not\Pent \neg \beta\text{, then }\phi^{\star}_{\alpha\land \beta}\Pent(\phi^{\star}_{\alpha})^{+}_{\beta} & \text{(subexpansion)}
\end{array}
\]
}


\nd As mentioned at the beginning of this section, every revision operation can be modelled by composing a contraction operation and an expansion one through the Levi Identity~\cite{Levi1977}: 
It is well-known that AGM contractions and AGM revisions are interconnected through the Levi Identity.

\begin{proposition}[\cite{AlchourronEtAl1985}]\label{th_repr_contr_rev}
    A revision $\star$ is an AGM revision operator iff it can be defined via the Levi Identity using an AGM contraction operator $\contr$.
\end{proposition}

\nd Applying  now the Levi identity we can immediately define the class of KM-revision operations $\star_{\km}$ as
\begin{equation}\label{eq_revision}
\boxed{ 
\phi^{\star_{\km}}_{\alpha}  = 
(\bigvee  \km\min(\phi\remaind\neg\alpha))\land\alpha
}
\end{equation}

\nd From Propositions~\ref{th_repr_contr} and \ref{th_repr_contr_rev}, we can conclude with
\begin{proposition}\label{th_repr_rev}
A function $\star$ is an AGM-revision operation iff it is a KM-revision operation.
\end{proposition}
\begin{example}\label{ex_revision}
 Consider Table \ref{Fig_1}. Assume we believe $(a\land b)$ and we are informed that $\neg b$ holds. Then the revision $(a\land b)^{\star_{\kms}}_{\neg b}$  corresponds to $(a\land b)^{\contr_{\kms}}_{b}\land \neg b$. 
 Since $(a\land b)^{\contr_{\kms}}_{b}\equiv a$, we can conclude that $(a\land b)^{\star_{\kms}}_{\neg b}\equiv (a\land \neg b)$.
\end{example}

\nd The analogue of Proposition~\ref{compreminderB} for the revision case is as follows.

\begin{proposition} \label{revsimple}
Consider a probability distribution $\pr$ over $\W$, a KM $\km$ and formulae $\phi$ and $\alpha$
Then
\begin{equation} \label{eqstarB}
{\phi}^{\star}_{\alpha} = \left \{ \begin{array}{ll}
\overl({\min_{\km}([\alpha]^+})) & \text{if } \phi\Pent\neg\alpha\\
\overl([\phi]^+\cap [\alpha]^+) & \text{otherwise.}
\end{array}\right .
\end{equation}
\end{proposition}


\nd Finally, likewise contraction, for expansion and revision operators $+$ and $\star$, respectively, we may also introduce related quantitative measures such as \emph{information gain of expansion} 
\begin{equation} \label{infog}
G^+_\pr(\phi, \alpha) =  \km(\phi + \alpha) - \km(\phi) \ ,
\end{equation}
\nd and  \emph{information change of revision} 

\begin{equation} \label{inforev}
R^\star_\pr(\phi, \alpha) =  \km(\phi \star \alpha) - \km(\phi) \ ,
\end{equation} 
\nd respectively. Clearly, while the former is non-negative, the latter may not. 

We conclude this section by providing the analogue of Corollary~\ref{compreminderBB}, both for information gain of expansion and information change of revision, whereas for the latter we use Proposition~\ref{revsimple}.

\begin{corollary} \label{exprevpreminderBB}
Consider a probability distribution $\pr$ over $\W$, a KM $\km$ and formulae $\phi$ and $\alpha$. 
Then
\begin{eqnarray}  
G^+_\pr(\phi, \alpha) & = &  
- \log_2 \pr(\alpha\mid \phi) \\ \label{eq_expAA} 
\nonumber \\ 
R^\star_\pr(\phi, \alpha) & = &
\left \{ \begin{array}{ll}
\log_2 \frac{\pr(\phi)}{\pr(\overl({\min_{\km}([\alpha]^+}))} & \text{if } \phi\Pent \neg \alpha\\
G^+_\pr(\phi, \alpha) & \text{otherwise . }
\end{array}\right . \label{eq_expAAA}
\end{eqnarray}

\end{corollary}



\subsection{Iterated belief revision in brief} \label{ibc}

\nd In this section, we also give a succinct look into the problem of \emph{iterated revision}~\cite{DarwichePearl1997}. That is, we are going to look at possible ways in which we can apply a sequence of revision operations in our framework. In~\cite{DarwichePearl1997} four well-known extra postulates for iterated revision have been proposed, which we adapt here to our probabilistic setting as follows:

\[
\begin{array}{lll}
{\bf (C1)}    &  \text{If }\psi\Pent \alpha\text{, then }(\phi^{\star}_{\alpha})^{\star}_{\psi}\equiv_{\pr} \phi^{\star}_{\psi}\\
{\bf (C2)}    &  \text{If }\psi\Pent \neg\alpha\text{, then }(\phi^{\star}_{\alpha})^{\star}_{\psi}\equiv_{\pr} \phi^{\star}_{\psi}\\
{\bf (C3)}    &  \text{If }\phi^{\star}_{\psi}\Pent \alpha\text{, then }(\phi^{\star}_{\alpha})^{\star}_{\psi}\Pent \alpha\\
{\bf (C4)}    &  \text{If }\phi^{\star}_{\psi}\not\Pent \neg\alpha\text{, then }(\phi^{\star}_{\alpha})^{\star}_{\psi}\not\Pent \neg\alpha\\
\end{array}
\]

\nd For an explanation of the intuitions behind such postulates we refer the reader to~\cite{DarwichePearl1997}. Here we just aim to show which of the iterated revisions postulates hold for KM-based belief revision.

The following can be shown.

\begin{proposition}\label{prop_iteratedC1}
Let $\star$ be a KM-based revision operator. Then $\star$ satisfies {\bf (C1), (C3)} and {\bf (C4)}.
\end{proposition}

\nd The next example shows that {\bf (C2)} is not always satisfied.

\begin{example} \label{exc2}





Let  $\Sigma=\{p,q\}$, and let 
$\pr = \{p\overline{q}^{0.6}, pq^{0.2}, \overline{p}\overline{q}^{0.1}, \overline{p}q^{0.1}\}$.
%
%
Clearly $\km(\overline{p}q)=\km(\overline{p}\overline{q})>\km(pq)>\km(p\overline{q})$. 
Now, let $\phi\equiv_\pr (p\land q)\lor(\neg p\land \neg q)$, $\psi\equiv_\pr p$, and $\alpha\equiv_\pr (\neg p\land q)$. 
It can be verified (see Table \ref{Fig_2}) that 
\begin{eqnarray*}
\psi & \Pent & \neg\alpha  \\
    \phi & \not\Pent & \neg \psi \\
    \phi & \Pent & \neg \alpha \\
\end{eqnarray*}

\begin{table}
\caption{Probability distribution $\pr$, and the relative KM, in Example \ref{exc2}.} \label{Fig_2}
\begin{center}
{
\begin{tabular}{|c|c|c|c|} \hline
$r$ & worlds & $\pr$ & $\kms$ \\ \hline
{$2$} & $\overline{p}q$ \hspace{0.2cm} $\overline{p}\overline{q}$ & $0.1$ & $3.32$ \\ \hline
{$1$} & $pq$ & $0.2$ & $2.32$\\ \hline
{$0$} & $p\overline{q}$ & $0.6$ & $0.74$ \\ \hline
\end{tabular}
}
\end{center}
\end{table}
%
%
By Proposition~\ref{revsimple}, the following results follow:
\begin{itemize}
    \item $[\phi^{\star}_{\psi}]=\overl([\phi]^+\cap [\psi]^+)=\{pq\}$;
    \item $[\phi^{\star}_{\alpha}]=\overl(\min_{\km}([\alpha]^+))=\{\overline{p}q\}$;
    \item $[(\phi^{\star}_{\alpha})^{\star}_{\psi}]=\overl(\min_{\km}([\psi]^+))=\{p\overline{q}\}$.
\end{itemize}

\nd Therefore, $(\phi^{\star}_{\alpha})^{\star}_{\psi} \not\equiv_\pr \phi^{\star}_{\psi}$.
    
\end{example}

\nd From Example~\ref{exc2} it follows that

\begin{proposition}\label{prop_iteratedC2}
There is a KM-based revision operator $\star$ that does not satisfy {\bf (C2)}.
\end{proposition}






\nd We recap that postulate  {\bf (C2)} formalises the idea that if $\psi$ contradicts the previous information $\alpha$, then $\psi$ completely erases the effect of the `temporary' presence of $\alpha$ in the belief set. However, the validity of such a postulate may be debatable, as the following example illustrates.



\begin{example}\label{ex_c2}
    Let $\Sigma=\{p,d,c\}$, where $p$, $d$, and $c$ read, respectively, `Bob owns a pet', `Bob owns a dog', and `Bob owns a cat'. Suppose now that, according to the statistics about pet ownership in Bob's country, we have the following probability distribution: 
    \[
    \pr=\{\bar{p}\bar{d}\bar{c}^{0.3},p\bar{d}c^{0.3},pd\bar{c}^{0.25},pdc^{0.15}\} \ .
    \]

    \nd Assume that our starting KB (our $\phi$) corresponds to the proposition $\neg p$, \ie~`Bob does not own a pet'. Also, our probability distribution $\pr$, assigning probability $0$ to some worlds, implies some background information: in particular,  $d\limp p, c\limp p$, and $p\limp (d\lor c)$ (for the sake of simplicity we assume that the only kinds of pets are dogs and cats).
    
    
    In the following, let $\alpha:=d$ and $\psi:= \neg d$. 
    \begin{itemize}
        \item  If we are informed that `Bob does not own a dog', \ie~we ask to revise $\phi$ by $\psi$,  by Eq.~\ref{eqstarB}, we `remain' in $\bar{p}\bar{d}\bar{c}$. That is, $\phi^{\star}_\psi \equiv_{\pr} \phi$. 

        \item  Instead, if we are informed that `Bob owns a dog',  \ie~we ask to revise $\phi$ by $\alpha$,  by Eq.~\ref{eqstarB}, we `move' to the world $pd\bar{c}$ (the most expected/least surprising world in which $d$ is satisfied) and, thus, $\phi^{\star}_\alpha \equiv_{\pr} \overl{pd\bar{c}}$.
        
        If successively we are informed that `Bob does not own a dog', \ie~we ask to revise $\phi^{\star}_\alpha$ by $\psi$,   by Eq.~\ref{eqstarB}, we have to consider both the worlds $\bar{p}\bar{d}\bar{c}$ and $p\bar{d}c$, which are the most expected/least surprising worlds in which Bob does not own a dog. Therefore, $(\phi^{\star}_\alpha)^{\star}_\psi \equiv_{\pr} \overl{\bar{p}\bar{d}\bar{c}} \lor \overl{p\bar{d}c} \equiv_{\pr} p \limp (\neg d \land c)$. That is, we leave open the possibility that Bob owns a cat.
    \end{itemize}

   \nd Therefore, $\phi^{\star}_\psi$ and $(\phi^{\star}_\alpha)^{\star}_\psi$ are not $\pr$-equivalent and we believe that the inequality is reasonable to hold in this case.
%
 %
    
\end{example}





\section{Conclusions, related and future work} \label{concl}

We have introduced novel quantitative Belief Change (BC) operators rooted in Knowledge Measures (KMs), aimed at minimising the (information theoretic) surprise of the information conveyed by the modified belief. 

In particular, our contributions encompass several key aspects:
\ii{i} we have presented a general information theoretic approach to KMs for which~\cite{Straccia20} is a special case;
\ii{ii} we've proposed KM-based BC operators that adhere to the AGM postulates; and
\ii{iii} we've demonstrated that any BC operator meeting the AGM postulates can be represented as a KM-based BC operator, regulated by a specific probability distribution over the worlds.
We have additionally introduced quantitative metrics that capture the loss of information during contraction, the acquisition of information through expansion, and the alteration of information during revision.  We also have given a succinct look into the problem of iterated revision~\cite{DarwichePearl1997}, and have also illustrated how one can construct a non-recovery-postulate-compliant contraction operator from our KM-based framework, by focusing on the so-called severe withdrawal~\cite{Rott99} model as illustrative example.

\paragraph{Closely related work}
Concerning KMs, to the best of our knowledge, we are not aware of other works, except~\cite{Straccia20}, that address the idea of KMs. However, there is an emerging community that deals with measuring inconsistency degrees, such as~\cite{Grant22,UlbrichtTB20,ThimmW19} (and references therein), which in fact, as pointed out in~\cite{Straccia20}, is an area were KMs may apply. 

Concerning, BC the area is by far much more investigated (see, \eg~\cite{FermeHansson2011,FermeHansson2018,Peppas08,AravanisP22} (and references therein). An in depth account of the huge literature is out of the scope of this paper and, thus, we refer here only to closely related work.
There have been some publications that have investigated possible connections between a probabilistic approach and belief revision~\cite{Lindstrom1989,Hawthornemakinson2007,Makinson2011,Spohn2012}, mainly showing the distance between probabilistic conditionals and qualitative conditionals, the latter based on belief revision. Such papers have a reasoning framework that is quite different \wrt~ours.
As mentioned within our paper, various definitions and properties in terms of $\pr$-entailment have a natural classical propositional logic counterpart~\cite{Caridroit17}, which we used to build up our framework. 
Worth mentioning are also~\cite{Levi91,Hansson95}, whose major objective was to illustrate a contraction that does not satisfy the recovery postulate.  Interestingly, besides the recovery postulate, Levi's contraction does neither satisfy conjunctive overlap nor conjunctive inclusion. But, the interesting point is that one may regain these two properties back by relying on the notion of \emph{Information Value} (IV)~\cite{Levi91,Hansson95}  that should be used to guide the selection of contraction candidates. Essentially, two variants of IVs  have been identified:\footnote{$\mathcal{V}(\alpha)$ is real-valued.} 
\ii{i} \emph{strong monotonicity} of IV $\mathcal{V}$: if $\alpha \models \beta$ then $\mathcal{V}(\alpha) > \mathcal{V}(\beta)$; and
\ii{ii} \emph{weak monotonicity} of IV $\mathcal{V}$: if $\alpha \models \beta$ then $\mathcal{V}(\alpha) \geq \mathcal{V}(\beta)$.
From that, the \emph{value-based Levi contraction} is then defined, roughly, as a disjunction of the $\mathcal{V}$-maximal elements in $S(\phi,\alpha)$,  where this latter is the set of so-called \emph{saturatable contractions} that consist of formulae $\psi$ weaker than $\phi$ such that $\psi \land \neg \alpha$ is maximally consistent.\footnote{In our setting, this set may be defined as the set  $S(\phi,\alpha) = \{ \psi \mid \phi\Pent\psi \text{ and } \psi \land \neg \alpha \text{ is maximal $\pr$-consistent } \}$,  where $\beta$ is is maximal if there is no $\pr$-consistent  $\gamma$ with $\gamma \PentL \beta$.}
Of course, the notion of IV resembles the axiom (E)  of KMs in~\cite{Straccia20} in the sense that the latter are weak monotonic IVs, though the opposite is not true.
Also, the overall definition of value-based Levi contraction is based on an IV maximisation selection process among saturatable contractions, while we rely on a  KM minimisation selection process over remainders (\wrt~a probabilistic distribution). Nevertheless, at the time of writing, the relationship between Levi's value-based Levi contraction and our KM-based one is still unclear and we will address it in future work,  together with other alternative constructions of BC operators. Let us also mention that we also tried to consider the $\km\max$ operator (with obvious definition) in place of the $\km\min$ operator in Definition~\ref{def_KMcontraction}, though it turned out to be questionable: as shown in Eq.~\ref{khconc}, KMs reward the less probable options: the less expected is a piece of information, the ``more information it contains'' from an information-theoretic point of view. But then, regarding Example~\ref{runexG}, if we select the remainders with highest KM, we would have given the priority to a world in which the non-flying birds would have been penguins (\ie~$bp\bar{o}\bar{f}w^{0.15}$) rather than to the more probable world $(b\bar{p}\bar{f}w)^{0.4}$.

\paragraph{Future Work} Topics for future research may include the following (besides those already mentioned in the paper here and there):
\ii{i} to investigate about different choices for KM postulates;
\ii{ii} to investigate and/or apply the notion of KMs in other contexts such as First-Order Language (FOL), or more in general  languages that do not have the finite-model property, to various forms of paraconsistent logics~\cite{Abe15}), \eg, measuring the degree of inconsistency~\cite{Grant22}), to non-monotonic logics~\cite{UlbrichtTB20,BrewkaNT08}, to conditional logics~\cite{Casini22}, and to uncertain, many-valued and mathematical fuzzy logics~\cite{Dubois01};
\ii{iii} to investigate on the application of KMs  to other approaches related to BC, inclusive to non recovery postulate-compliant variants; and
\ii{iv} to investigate forms of belief change obtained by combining the use of KMs with some kind of measure of similarity/semantic closeness among worlds, such as Dalal distance~\cite{Dalal88}.







\section*{Acknowledgment}
\nd This research was partially supported by TAILOR, a project funded by EU Horizon 2020 research and innovation programme under (GA No 952215). This paper is also supported by the FAIR (Future Artificial Intelligence Research) project funded by the NextGenerationEU program within the PNRR-PE-AI scheme (M4C2, investment 1.3, line on Artificial Intelligence). Eventually, this work has also been partially supported by the H2020 STARWARS Project (GA No. 101086252), type of action HORIZON TMA MSCA Staff Exchanges.


\appendix


\section{Some Proofs} \label{proofs}

\setcounter{proposition}{\getrefnumber{prsigmapB}-1}
\begin{proposition} 
For a formula $\phi$, $\Sigma_\phi \subseteq \bar{\Sigma} \subseteq \Sigma$ and a probability distribution $\pr_{\bar{\Sigma}}$, we have that
$\pr_\Sigma(\phi)  =   \pr_{\bar{\Sigma}}(\phi)$,
where $\pr_\Sigma$ is the extension of $\pr_{\bar{\Sigma}}$.
\end{proposition}
\begin{proof}
By definition and Eq.~\ref{sigmapbisB} we have
\[
\pr_{\bar{\Sigma}}(\phi)  =   \sum_{\bar{w} \in [\phi]_{\bar{\Sigma}}} \pr_{\bar{\Sigma}}(\bar{w}) 
= \sum_{\bar{w}\in [\phi]_{\bar{\Sigma}}} \sum_{w \in \extw{\bar{w}}_{\Sigma}} \pr_\Sigma(w) 
 =  \sum_{w\in [\phi]_\Sigma}  \pr_\Sigma(w) = \pr_\Sigma(\phi) \ ,
\]
which concludes.
\end{proof}

\setcounter{proposition}{\getrefnumber{propextconf}-1}
\begin{proposition}
Consider formulae $\phi$ and $\psi$, $\Sigma_\phi \cup \Sigma_\psi \subseteq \bar{\Sigma} \subseteq \Sigma$,
a probability distribution $\pr$ over $\bar{\Sigma}$, and the extension $\pr_{\Sigma}$  of  $\pr$  to $\Sigma$. Then for $\lhd\in\{\leq,<\}$, $\phi \Pentlhd \psi$ iff $\phi \PentPlhd{\pr_{\Sigma}} \psi$.
\end{proposition}
\begin{proof}

By Eq.~\ref{sigmapB}, for  $\bar{w} \in \W_{\bar{\Sigma}}$ and an extension $w \in \extw{\bar{w}}_{\Sigma}$, we have $\pr(\bar{w})  = \pr_{\Sigma}(w)$ and, thus, $\phi \land \neg \pr0$ and 
$\phi \land \neg {\pr_{\Sigma}}0$ are equivalent. Therefore, 
$\phi \Pent \psi$ iff $\phi \PentP{\pr_{\Sigma}} \psi$ holds. 
Furthermore, the strict case follows from the fact that we rule out in $\phi \land \neg \pr0$ and 
$\phi \land \neg {\pr_{\Sigma}}0$ all worlds with $0$ probability (see also Remark~\ref{nonz}), which concludes.
\end{proof}

\setcounter{proposition}{\getrefnumber{probent}-1}
\begin{proposition} 
    Consider formulae $\phi$ and $\psi$ \wrt~$\Sigma$, and a probability distribution $\pr$ over $\Sigma \supseteq \Sigma_{\phi}\cup\Sigma_{\psi}$.
    If $\phi \Pentlhd \psi$ then $\pr(\phi) \lhd \pr(\psi)$, for  $\lhd\in\{\leq,<\}$.
\end{proposition}

\begin{proof}
%
Consider the case $\phi \Pent \psi$. 
By definition of $\pr$-entailment, $[\phi \land \neg \pr0]_\Sigma \subseteq [\psi]_\Sigma$ holds. Therefore, as we rule out $0$ probability worlds,  $\pr(\phi) = \pr(\phi \land \neg \pr0) \leq \pr(\psi)$.

Similar to Proposition~\ref{propextconf}, the strict case follows from the fact that we rule out in $\phi \land \neg \pr0$ all worlds with $0$ probability (see also Remark~\ref{nonz}), which concludes.
\end{proof}

\setcounter{proposition}{\getrefnumber{kmsprop}-1}
\begin{proposition} 
 A KM $\km$ satisfying (KM1) - (KM3) is of the form $\frac{1}{\log_2 b} \cdot \kms$, for a constant $b>1$, 
 \ie~for a formula $\phi$ \wrt~$\Sigma$ and a probability distribution $\pr$ over $\Sigma$,
 $\km(\phi) = -\log_b \pr(\phi) = \frac{1}{\log_2 b} \cdot \kms (\phi)$. Also, $\kms$ satisfies (KM1) - (KM3).
\end{proposition}
\begin{proof}[Proof]
For the interested reader, we outline here a proof. 
Let us consider the function 
$I(\pr(\phi)) = \kms(\phi)$, \ie~$I$ has a probability as parameter. Then, for a twice differentiable function $I$,\footnote{We make this assumption for ease of presentation, though one doesn't need to do so~\cite{Shannon48}.} we have 
by (KM3),  $I(p  p') =  I(p) + I(p')$. Taking the derivative \wrt~$p$, we get
$p' I'(p  p') =   I'(p)$. Now, taking the derivative \wrt~$p'$, 
$I'(p p') + pp' I''(p p') =   0$ follows and by introducing $p=p p'$, we get 
$I'(p) + p I''(p) =   0$. That is, $(p I'(p))' =  0$.
This differential equation has solution $I(p) = k\ln p + c$ for $k,c \in \mathbb{R}$. Now, by (KM1), $I(1) = 0$ and, thus,
$c=0$. (KM2) tells us that $I(p)$ is monotonically non-increasing and, thus,  together with $I(1) = 0$ we have $I(p) \geq 0$ for all $p \in[0,1]$. Therefore, $k<0$ has to hold. Eventually, choosing a value for $k$ is equivalent to choosing a value 
$b>1$ for $k = -\frac{1}{\ln b}$, \ie~$b$ is the value of the base of the logarithm. Therefore, $I(p) = -\log_b p$ and, thus, $\km(\phi) = -\log_b \pr(\phi)$. Eventually, by choosing $b=2$, we also get that $\kms$ satisfies (KM1) - (KM3), which concludes.
%
\end{proof}


\setcounter{proposition}{\getrefnumber{kmsclassical}-1}
\begin{proposition} 
For any formula $\phi$, given the uniform probability distribution $\pr^u_{\Sigma}$ over $\Sigma$,
$\kms(\phi) = \kmh(\phi)$.
\end{proposition}
\begin{proof}
If $\phi$ unsatisfiable then  $\pr^u_{\Sigma}(\phi) = 0$ and, thus, by def.~$\kms(\phi) = \infty = \kmh(\phi)$.
Otherwise, $\pr^u_{\Sigma} = 1/2^{|\Sigma|}$ and, thus, its marginalisation  to $\Sigma_\phi$ is 
$\pr^u_{\Sigma_\phi} = 1/2^{|\Sigma_\phi|}$ and $\pr^u_\Sigma(\phi) = \pr^u_{\Sigma_\phi}(\phi) = |[\phi]_{\Sigma_\phi}|/2^{|\Sigma_\phi|}$ holds (by Eq.~\ref{prsphiB}). So,
$\kms(\phi)$  $=  - \log_2 \pr_\Sigma(\phi)$ 
$=  - \log_2 \pr_{\Sigma_\phi}(\phi)$
$=  - \log_2 \frac{|[\phi]_{\Sigma_\phi}|}{2^{|\Sigma_\phi|}}$ 
$=  |\Sigma_\phi| - \log_2 |[\phi]_{\Sigma_\phi}|$ 
$=  \kmh(\phi)$, which concludes.
\end{proof}

\setcounter{proposition}{\getrefnumber{compreminder}-1}
\begin{proposition} 
Consider a probability distribution $\pr$ over $\W_\Sigma$ and formulae $\phi$ and $\alpha$ such that
$\phi \Pent \alpha$ and $\top \not\Pent \alpha$. Then
$\psi \in \phi\remaind\alpha$ iff 
%
\begin{equation*}
\psi \equiv \overl({[\phi]_\Sigma^+ \cup \{w\}}) 
\end{equation*}
\nd for some $w \in [\neg \alpha]_\Sigma^+$. Furthermore, 
$\pr(\psi) = \pr(\overl{([\phi]_\Sigma^+})) + \pr(w)$. 
On the other hand, if $\phi \not \Pent \alpha$ or $\top \Pent \alpha$ then 
 $\phi\remaind\alpha = \{\overl{([\phi]^+})\}$.
\end{proposition}
\begin{proof}
Assume $\phi \Pent \alpha$ and $\top \not\Pent \alpha$. Clearly, $\overl({[\phi]_\Sigma^+ \cup \{w\}})$ is a possible reminder as $\phi \Pent \overl({[\phi]_\Sigma^+ \cup \{w\}})$ and $\overl({[\phi]_\Sigma^+ \cup \{w\}})\not\Pent\alpha$ holds. Now, assume to the contrary that $\overl({[\phi]_\Sigma^+ \cup \{w\}})$ is not a reminder. Therefore, there is another $\psi \in \phi\remaind\alpha$ \st~$\psi\PentL  \overl({[\phi]^+ \cup  \{w\}})$. But, that to be the case, it must be $[\phi]^+\subset [\psi]^+\subset ([\phi]^+ \cup \{w\})$, that cannot be the case.
    Therefore, each reminder corresponds to a formula $\overl({[\phi]_\Sigma^+ \cup \{w\}})$ with $w \in [\neg \alpha]_\Sigma^+$, which concludes this part.
Additionally, as $w\not \in [\phi]_\Sigma^+$, $\pr(\psi) = \pr(\overl{([\phi]_\Sigma^+})) + \pr(w)$ holds.

Eventually, if $\phi \not \Pent \alpha$ or $\top \Pent \alpha$ holds then 
 $\phi\remaind\alpha = \{\overl{([\phi]^+})\}$ by the definition of the reminder set, which concludes the proof.
\end{proof}

\setcounter{proposition}{\getrefnumber{compreminderB}-1}
\begin{proposition} 
Consider a probability distribution $\pr$ over $\W$, a KM $\km$ and formulae $\phi$ and $\alpha$. 
Then
%

\begin{equation*} 
\phi^{\contrk}_{\alpha} \equiv
\left \{ \begin{array}{ll}
\overl({[\phi]^+ \cup \min_{\km}([\neg \alpha]^+})) & \text{if } \phi\Pent \alpha\\
\overl{[\phi]^+} & \text{otherwise, }
\end{array}\right .
\end{equation*}

\nd where $\min_{\km}([\neg \alpha]^+)=\{w\in\W \mid w\in[\neg \alpha]^+$ and there is no $w'\in\W\text{ s.t. } w'\in[\neg \alpha]^+ \text{ and } \km(w') < \km(w)\}$.
\end{proposition}
\begin{proof} 
If $\phi \not \Pent \alpha$ then by Proposition~\ref{compreminder}, $\phi\remaind\alpha = \{\overl{([\phi]^+})\}$ and, thus, $\phi^{\contrk}_{\alpha} = \overl({[\phi]^+})$.
Now, let us assume $\phi \Pent \alpha$ instead. There are two cases.

\begin{description}
    \item[Case $\top \Pent \alpha$.]  If  $\top \Pent \alpha$, then  $\min_{\km}([\neg \alpha]^+) = \emptyset$.  Again, by Proposition~\ref{compreminder},
$\phi\remaind\alpha = \{\overl{([\phi]^+})\}$ and, thus, $\phi^{\contrk}_{\alpha} = \overl({[\phi]^+}) = \overl({[\phi]^+ \cup \emptyset })$.

    \item[Case $\top \not\Pent \alpha$.] 
    Let us prove that $\overl({[\phi]^+ \cup  \min_{\km}([\neg \alpha]^+}))$ $\equiv$ $\bigvee  \km\min(\phi\remaind\alpha)$. It is sufficient to prove that each formula in $\km\min(\phi\remaind\alpha)$ corresponds
to a formula $\overl({[\phi]^+ \cup \{ w \} })$ with $w\in\min_{\km}([\neg \alpha]^+)$. 
    
At first we prove that a formula $\overl({[\phi]^+ \cup  \{w\}})$ with $w\in\min_{\km}([\neg \alpha]^+)$ is in $\km\min(\phi\remaind\alpha)$. 
Now, $\overl({[\phi]^+ \cup  \{w\}})$, is a reminder by Proposition~\ref{compreminder},
%
\ie~$\overl({[\phi]^+ \cup \{w\}}) \in\phi\remaind\alpha$, and we have to prove that $\overl({[\phi]^+ \cup \{w\}}) \in\km\min(\phi\remaind\alpha)$. Assume it is not the case. That implies that there is $\psi\in\phi\remaind\alpha$ s.t.~$\km(\psi) < \km(\overl({[\phi]^+ \cup \{w\}}))$ and, thus, $\pr(\psi) > \pr(\overl({[\phi] \cup \{w\}})) = \pr(\overl{[\phi])} + \pr(w)$ (by Proposition~\ref{compreminder}).  But, $\psi\in\phi\remaind\alpha$ means that, by  Proposition~\ref{compreminder}, $\psi \equiv \overl({[\phi]^+ \cup \{w'\}}))$ for some $w' \in [\neg \alpha]^+$, $\pr(\psi) = \pr(\overl{[\phi]}) + \pr(w')$ and, thus, $\pr(w') > \pr(w)$, that cannot be the case as $w\in \min_k([\neg \alpha]^+)$ and, thus, $\pr(w) \geq \pr(w')$ has to hold.  
    
To prove that every formula in $\km\min(\phi\remaind\alpha)$ indeed corresponds to some formula
$\overl({[\phi]^+ \cup \{w\}})$, with $w\in\min_{\km}([\neg \alpha]^+)$, consider that every formula in $\km\min(\phi\remaind\alpha)$ must correspond to a formula $\overl({[\phi]^+ \cup \{w\}})$ with $w\in[\neg \alpha]^+$, and for it to be in $\km\min(\phi\remaind\alpha)$ it must be one of such formulae with the highest probability \wrt~$\pr$, hence $w$ must be in $\min_{\km}([\neg \alpha]^+)$.
    
Therefore, $\km\min(\phi\remaind\alpha)$ $ = $ $\{ \overl({[\phi]^+ \cup \{w\}}) \mid w\in $ $\min_{\km}([\neg \alpha]^+)\}$. That is, $\overl({[\phi]^+ \cup  \min_{\km}([\neg \alpha]}))$ $\equiv$ $\bigvee \km\min(\phi\remaind\alpha)$, which concludes.
    
\end{description}

\end{proof}

\setcounter{corollary}{\getrefnumber{compreminderBB}-1}
\begin{corollary} 
Consider a probability distribution $\pr$ over $\W$, a KM $\km$ and formulae $\phi$ and $\alpha$. 
Then
\begin{equation*} \label{eq_contrALoss}
L^\contr_\pr(\phi, \alpha) = 
\left \{ \begin{array}{ll}
\log_2 (1+ \frac{\pr(\overl\min_{\km}([\neg \alpha]^+)}{\pr(\phi)}) & \text{if } \phi\Pent \alpha\\
0 & \text{otherwise . }
\end{array}\right .
\end{equation*}
\end{corollary}
\begin{proof}
Consider Proposition~\ref{compreminderB} and Eq.~\ref{losscontra}. If  $\phi\not\Pent \alpha$ then $L^\contr_\pr(\phi, \alpha) = \km(\phi) - \km(\overl{([\phi]^+})) = \km(\phi) - \km(\phi) = 0$. Otherwise ($\phi\Pent \alpha$), let us note that $\pr(\overl({[\phi]^+ \cup \min_{\km}([\neg \alpha]^+}))) = \pr(\phi) + \pr(\min_{\km}([\neg \alpha]^+) = \pr(\phi) \cdot(1 + \sfrac{\pr(\min_{\km}([\neg \alpha]^+)}{\pr(\phi)})$ and, thus,  $L^\contr_\pr(\phi, \alpha) = \km(\phi) - \km(\phi^{\contr}_{\alpha})$ $= -log_2 \pr(\phi) + \log_2 \pr(\overl({[\phi]^+ \cup \min_{\km}([\neg \alpha]^+})))  = -log_2 \pr(\phi) + \log_2 \pr(\phi) \cdot(1 + \sfrac{\pr(\min_{\km}([\neg \alpha]^+)}{\pr(\phi)})$ 
$ = -log_2 \pr(\phi) + \log_2 \pr(\phi) + \log_2 (1 + \sfrac{\pr(\min_{\km}([\neg \alpha]^+)}{\pr(\phi)})$, which concludes.
\end{proof}

\setcounter{proposition}{\getrefnumber{prop_AGM_sound}-1}
\begin{proposition}
A KM-contraction operator $\contrk$ is an AGM contraction operator, that is, it satisfies postulates $(\contr 1)$-$(\contr 7)$.
\end{proposition}

\begin{proof}
We have to prove that all AGM postulates are satisfied:
\begin{description}
    \item[$(\contr 1)$.] $\phi\Pent\psi$ for every $\psi\in\phi\remaind\alpha$, hence $\phi\Pent\bigvee \km\min(\phi\remaind\alpha)$.
    
    \item[$(\contr 2)$.] If $\phi\not\Pent \alpha$, then $\km\min(\phi\remaind\alpha)$ contains only $\overl{([\phi]^+})$, which is $\pr$-equivalent to $\phi$.

    \item[$(\contr 3)$.] Assume $\top \not\Pent \alpha$. Then $\psi\not\Pent\alpha$ for all $\psi\in\phi\remaind\alpha$, that, by the monotonicity of $\Pent$, implies $\bigvee\km\min(\phi\remaind\alpha)\not\Pent\alpha$.
    
    \item[$(\contr 4)$.] It is immediate from the definition of $\phi\remaind\alpha$.

   \item[$(\contr 5)$.] It suffices to prove that $\phi^{\contrk}_{\alpha}\Pent\alpha\limp\phi$. To this end, we prove that for all $\psi\in\km\min(\phi\remaind\alpha)$, $\psi\Pent\alpha\limp\phi$. 
    So, assume that it is not the case. Clearly,
    $\psi\land(\alpha\limp\phi)\Pent\psi$. But,
    $\psi\not\Pent\alpha\limp\phi$ and, thus, 
    $\psi \not \Pent \psi\land(\alpha\limp\phi)$. 
    Therefore, by Proposition~\ref{prKM},  $\km(\psi\land(\alpha\limp\phi)) > \km(\psi)$. 
    As $\psi\in\km\min(\phi\remaind\alpha)$, 
    $\psi\land(\alpha\limp\phi)\notin\phi\remaind\alpha$ follows. But,  $\phi\Pent\psi\land(\alpha\limp\phi)$,
    so $\psi\land(\alpha\limp\phi)\Pent\alpha$. That is, $\psi\Pent(\alpha\limp\phi)\limp\alpha$. Since $(\alpha\limp\phi)\limp\alpha\equiv\alpha$, we conclude $\psi\Pent\alpha$, against the assumption that $\psi\in\phi\remaind\alpha$.
    %
    %
    Therefore, for all $\psi\in\km\min(\phi\remaind\alpha)$, $\psi\Pent\alpha\limp\phi$, and consequently $\phi^{\contrk}_{\alpha}\Pent\alpha\limp\phi$.
    
    \item[$(\contr 6)$.] 
    By definition $\phi^{\contrk}_{\alpha\land\beta} = \bigvee \km\min(\phi\remaind(\alpha\land\beta))$, while $\phi^{\contrk}_{\alpha} = \bigvee \km\min(\phi\remaind\alpha)$ and $\phi^{\contrk}_{\beta} = \bigvee \km\min(\phi\remaind\beta)$. 
   %
So, if $\psi\in\phi\remaind(\alpha\land\beta)$ then either $\psi\not\Pent\alpha$ or $\psi\not\Pent\beta$, and that there is no $\psi'\PentL \psi$
s.t.~$\phi\Pent\psi'$, $\psi'\Pent\psi$, and either $\psi'\not\Pent\alpha$ or $\psi\not\Pent\beta$, respectively. That is, $\psi\in \phi\remaind\alpha$ or $\psi\in \phi\remaind\beta$. 
Hence, for all $\psi\in\phi\remaind(\alpha\land\beta)$, either $\psi\in\phi\remaind\alpha$ or $\psi\in\phi\remaind\beta$. This leaves us with three possible options:  $\km\min(\phi\remaind(\alpha\land\beta))=\km\min(\phi\remaind\alpha)$, $\km\min(\phi\remaind(\alpha\land\beta))=\km\min(\phi\remaind\beta)$, or $\km\min(\phi\remaind(\alpha\land\beta))=\km\min(\phi\remaind\alpha)\cup \km\min(\phi\remaind\beta)$, depending whether, given any $\psi\in \km\min(\phi\remaind\alpha)$ and any $\psi'\in \km\min(\phi\remaind\beta)$, $k(\psi') > \km(\psi)$, $\km(\psi')<k(\psi)$, or $\km(\psi) = \km(\psi')$, respectively. In any of such three cases, $(\contr 6)$ is satisfied.

    \item[$(\contr 7)$.] If  $\phi^{\contr}_{\alpha\land \beta}\not\Pent \alpha$, then there is a $\gamma\in \km\min(\phi\remaind(\alpha\land\beta))$ s.t. $\gamma\not\Pent\alpha$. This implies that all the $\gamma\in\km\min(\phi\remaind\alpha)$ are in $\km\min(\phi\remaind(\alpha\land\beta))$. That is, $\km\min(\phi\remaind\alpha)\subseteq \km\min(\phi\remaind(\alpha\land\beta))$. This in turn implies that $\bigvee \km\min(\phi\remaind\alpha)\entails\bigvee \km\min(\phi\remaind(\alpha\land\beta))$. Therefore,  
    $\phi^{\contr}_{\alpha}\Pent\phi^{\contr}_{\alpha\land \beta}$ holds.
\end{description}
\nd This concludes the proof.
\end{proof}

\setcounter{proposition}{\getrefnumber{prop_faith_prob}-1}
\begin{proposition}
For any formula $\phi$, set of worlds $\W$, and $\phi$-faithful ranking $r$ over $\W$, it is always possible to define an $r$-faithful probability distribution $\pr$ over $\W$.
\end{proposition}
\begin{proof}
    It is sufficient to define a procedure to determine $\pr$ given any $r$. For example, we may define a function $f:\W\mapsto \mathbb{N}$ with $f(w)=r(w)+1$ and then define $\pr(w)$ as follows: let $m = 1+ \max_{w'\in\W} r(w')$. Then
    $\pr(w)= \frac{m - r(w)}{\sum_{w'\in\W}f(w')}$.
    It is easy to check that $\pr$ is an $r$-faithful probability distribution, that is,  $\sum_{w\in\W}\pr(w)=1$ and that the two conditions for $r$-faithfulness are satisfied.
\end{proof}

\setcounter{proposition}{\getrefnumber{prop_AGM_compl}-1}
\begin{proposition}
Let $\phi$ be any formula and $\contr$ any AGM-contraction for $\phi$. We can define a KM contraction $\contr_{\km}$ s.t., for all $\alpha$,
$\phi^\contr_\alpha \equiv \phi^{\contr_{\km}}_\alpha$ holds.
\end{proposition}
\begin{proof} 
    Given $\phi$ and $\contr$, by Proposition~\ref{KMranking} we know that there is a faithful ranking $r$ over a set of worlds $\W$ s.t. the contraction $\contr_r$ associated to $r$ (see Eq.\ref{eq_ranked_contr}) corresponds to $\contr$, that is, for any $\alpha$, $\phi^{\contr}_{\alpha} \equiv \phi^{\contr_r}_{\alpha}$.
    We now define an $r$-faithful probability distribution $\pr$ over $\W$, and on top of it we generate a KM $\km$. Proposition~\ref{prop_faith_prob} guarantees that such a probability distribution exists.
%
    
    Now, we have to prove that $\contr_r$ and $\contr_{\km}$ correspond to each other. 
If $\phi \not \leq \alpha$ then $\phi \not \Pent \alpha$, as there are no zero probability worlds,  and, thus, by the vacuity postulate ${\phi}^{\contr_{r}}_{\alpha} \equiv \phi \equiv_\pr \phi^{\contrk}_\alpha$. Otherwise, clearly, for any $w,w'\in\W$, $r(w)\leq r(w')$ iff $\pr(w)\geq \pr(w')$ iff $\km(w)\leq\km(w')$, by (KM2). Consequently, ${\phi}^{\contr_{r}}_{\alpha} \equiv \overl({[\phi] \cup    \min_{r}([\neg \alpha]}))$ means that ${\phi}^{\contr_{r}}_{\alpha} \equiv \overl({[\phi]^+ \cup    \min_{\km}([\neg \alpha]^+})) \equiv \phi^{\contrk}_\alpha$, by Proposition~\ref{compreminderB}, which concludes.
\end{proof}

\setcounter{corollary}{\getrefnumber{compreminderBBS}-1}
\begin{corollary} 
Consider a probability distribution $\pr$ over $\W$, formulae $\phi$ and $\alpha$. 
Then
\begin{equation*}
L^{\contrks}_\pr(\phi, \alpha) = 
\left \{ \begin{array}{ll}
\log_2 \frac{\pr(\overl\sigma(\neg \alpha)}{\pr(\phi)} & 
\text{if }  \top \not \Pent \alpha \\
0 & \text{otherwise . }
\end{array}\right .
\end{equation*}
\end{corollary}
\begin{proof}
The case $\top\Pent \alpha$ is straightforward. So, let us consider the case $\top \not \Pent \alpha$. 
By definition of the loss and Proposition~\ref{propseveresigma}, we have 
$L^{\contrks}_\pr(\phi, \alpha) = -\log_2 \pr(\phi) + \log_2 \pr(\sigma(\neg \alpha)) =\log_2 \frac{\pr(\overl\sigma(\neg \alpha)}{\pr(\phi)}$, which concludes.
\end{proof}

\setcounter{proposition}{\getrefnumber{prop_SW_sound}-1}
\begin{proposition}
Consider a probability distribution $\pr$ and a KM-contraction operator $\contrk$ based on $\pr$. Then the function $\contrks$ is a severe withdrawal operator, that is, it satisfies postulates $(\contrs 1)$-$(\contrs 4)$, $(\contrs 6a)$ and $(\contrs 7)$.
\end{proposition}
\begin{proof}
    We have to prove that all postulates of severe withdrawal are satisfied by $\contrks$:
\begin{description}
    \item[$(\contrs 1)$.] Follows immediately from Proposition~\ref{propseveresigma}.
    
    \item[$(\contrs 2)$.] If $\phi\not\Pent \alpha$ the $\sigma(\neg \alpha) = [\phi]^+$. Therefore, if $\phi\not\Pent \alpha$ or 
    $\top \Pent \alpha$, then, by Proposition~\ref{propseveresigma}, $[\phi^{\contrks}_\alpha]^+ = [\phi]^+$ and, thus,  $\phi^{\contrks}_{\alpha}\equiv_{\pr} \phi$, which concludes.
    
    \item[$(\contrs 3)$.] If $\top \not\Pent \alpha$ then,  by Proposition~\ref{propseveresigma}, $[\phi^{\contrks}_\alpha]^+ = \sigma(\neg \alpha) \neq \emptyset$ and, thus, $\phi^{\contrks}_{\alpha}\not\Pent \alpha$,  which concludes.
    
    \item[$(\contrs 4)$.] If $\alpha\equiv_{\pr} \beta$, then $[\alpha]^+ = [\beta]^+$, \ie~$[\neg \alpha]^+ = [\neg \beta]^+$.
    Therefore,  $\top \Pent \alpha$ iff $\top \Pent \beta$. So, if $\top \Pent \alpha$ then $\top \Pent \beta$ and, thus,  by Proposition~\ref{propseveresigma}, $[\phi^{\contrks}_\alpha]^+ = [\phi]^+ = [\phi^{\contrks}_\beta]^+$. 
    On the other hand, if $\top \not\Pent \alpha$ then $\top \not\Pent \beta$ and, thus, by Proposition~\ref{propseveresigma}, $[\phi^{\contrks}_\alpha]^+ = \sigma(\neg \alpha) = \sigma(\neg \beta) = [\phi^{\contrks}_\beta]^+$. Therefore, in any case $[\phi^{\contrks}_\alpha]^+ = [\phi^{\contrks}_\beta]^+$ and, thus, 
     $\phi^{\contrks}_{\alpha}\equiv_{\pr} \phi^{\contrks}_{\beta}$, which concludes.

    \item[$(\contrs 6a)$.] Assume $\top \not\Pent \alpha$ and, thus, $\top \not\Pent \alpha \land \beta$. By Proposition~\ref{propseveresigma} we have $[\phi^{\contrks}_\alpha]^+ = \sigma(\neg \alpha)$ and 
    $[\phi^{\contrks}_{\alpha\land \beta}]^+ = \sigma(\neg \alpha \lor \neg \beta)$. If $\top \Pent \beta$, then 
    $\sigma(\neg \alpha \lor \neg \beta) = \sigma(\neg \alpha)$ and, thus, $[\phi^{\contrks}_{\alpha\land \beta}]^+ = [\phi^{\contrks}_\alpha]^+$. In particular, $\phi^{\contrks}_{\alpha\land \beta}\Pent\phi^{\contrks}_{\alpha}$ holds.
    Assume now $\top \not\Pent \beta$ instead. 
    We have two cases: 
    \ii{i} $\sigma(\neg \alpha) \subseteq \sigma(\neg \beta)$; or
    \ii{ii} $\sigma(\neg \beta) \subset \sigma(\neg \alpha)$. By Proposition~\ref{propseveresigma}, in case \ii{i} we have 
    $[\phi^{\contrks}_{\alpha\land \beta}]^+ = \sigma(\neg \alpha) = [\phi^{\contrks}_{\alpha}]^+$.
    In case \ii{ii} we have $[\phi^{\contrks}_{\alpha\land \beta}]^+ = \sigma(\neg \beta) \subset \sigma(\neg \alpha) = [\phi^{\contrks}_{\alpha}]^+$. 
    Therefore, in any case $[\phi^{\contrks}_{\alpha\land \beta}]^+ \subseteq [\phi^{\contrks}_{\alpha}]^+$ holds, 
    \ie~$\phi^{\contrks}_{\alpha\land \beta}\Pent\phi^{\contrks}_{\alpha}$, which concludes.

     \item[$(\contrs 7)$.] Assume $\phi^{\contrs}_{\alpha\land \beta}\not\Pent \alpha$. Therefore,  
     $\top \not\Pent \alpha$ and, thus, $\top \not\Pent \alpha \land \beta$.
     By Proposition~\ref{propseveresigma}, $[\phi^{\contrks}_{\alpha}]^+ = \sigma(\neg \alpha)$ and 
     $[\phi^{\contrks}_{\alpha\land \beta}]^+ = \sigma(\neg \alpha \lor \neg \beta)$. 
    If $\top \Pent \beta$, then $\sigma(\neg \alpha \lor \neg \beta) = \sigma(\neg \alpha)$ and, thus, 
    $[\phi^{\contrks}_{\alpha\land \beta}]^+ = [\phi^{\contrks}_\alpha]^+$. So, assume now $\top \not\Pent \beta$ instead.
     From $\phi^{\contrs}_{\alpha\land \beta}\not\Pent \alpha$ we also know that $\sigma(\neg \beta) \subset \sigma(\neg \alpha)$ can not be the case as otherwise $[\phi^{\contrks}_{\alpha\land \beta}]^+ = \sigma(\neg \beta) \subseteq [\alpha]^+$, violating our assumption. Therefore, $\sigma(\neg \alpha) \subseteq \sigma(\neg \beta)$ has to hold and, thus, 
     $[\phi^{\contrks}_{\alpha\land \beta}]^+ = \sigma(\neg \alpha) = [\phi^{\contrks}_{\alpha}]^+$.
      Consequently, in any case $\phi^{\contrks}_{\alpha\land \beta}\Pent\phi^{\contrks}_{\alpha}$ holds, which concludes.
\end{description}
\end{proof}

\setcounter{proposition}{\getrefnumber{th_repr_contrs}-1}
\begin{proposition}
A function $\contrs$ is a severe withdrawal operator iff it is a KM-severe withdrawal function operation. 
\end{proposition}
\begin{proof}
    Proposition \ref{prop_SW_sound} proves that every KM-severe withdrawal function $\contrks$ is a severe withdrawal operator, that is, it satisfies the postulates $(\contrs 1)$-$(\contrs 4)$, $(\contrs 6a)$ and $(\contrs 7)$.

    We need to prove the opposite direction, and the proof is quite straightforward, given the results in \cite{RottPagnucco1999}, in particular Observation 15, which proves that for every severe withdrawal operator there is a correspondent withdrawal function defined on a system of spheres in the way shown in Figure \ref{figspheresevere}. Given such a system of spheres, it is sufficient to define a probability distribution over the worlds in the same way we have done in the proof of Proposition \ref{prop_faith_prob}, and it is easy to prove that such a probability distribution and the correspondent KM would generate the same withdrawal function.
\end{proof}

\setcounter{proposition}{\getrefnumber{revsimple}-1}
\begin{proposition} 
Consider a probability distribution $\pr$ over $\W$, a KM $\km$ and formulae $\phi$ and $\alpha$
Then
\begin{equation*} 
{\phi}^{\star}_{\alpha} = \left \{ \begin{array}{ll}
\overl({\min_{\km}([\alpha]^+})) & \text{if } \phi\Pent\neg\alpha\\
\overl([\phi]^+\cap [\alpha]^+) & \text{otherwise.}
\end{array}\right .
\end{equation*}
\end{proposition}
\begin{proof}
We have to consider two cases: 
\begin{enumerate}
    \item if $\phi\Pent\neg\alpha$, then using Proposition~\ref{compreminderB}
    \begin{eqnarray*}
    {\phi}^{\star}_{\alpha} & = & (\bigvee  \km\min(\phi\remaind\neg\alpha))\land \alpha \\
    & \equiv & \overl(({[\phi]^+ \cup   \min_{\km}[\alpha]^+})\cap [\alpha]^+) \\
    & \equiv & \overl({\min_{\km}([\alpha]^+})) \ . 
    \end{eqnarray*}

    \item If $\phi\not\Pent\neg\alpha$, then by Proposition~\ref{compreminder}, 
    $\phi\remaind\neg\alpha = \{\overl{([\phi]^+}\}$ and, thus,  
    ${\phi}^{\star}_{\alpha} = (\bigvee  \km\min(\phi\remaind\neg\alpha))\land\alpha \equiv \overl({[\phi]^+ \cap    [\alpha]^+})$, which concludes.    
\end{enumerate}
\end{proof}

\setcounter{corollary}{\getrefnumber{exprevpreminderBB}-1}
\begin{corollary} 
Consider a probability distribution $\pr$ over $\W$, a KM $\km$ and formulae $\phi$ and $\alpha$. 
Then
\begin{eqnarray*}  
G^+_\pr(\phi, \alpha) & = &  
- \log_2 \pr(\alpha\mid \phi) \\ 
\nonumber \\ 
R^\star_\pr(\phi, \alpha) & = &
\left \{ \begin{array}{ll}
\log_2 \frac{\pr(\phi)}{\pr(\overl({\min_{\km}([\alpha]^+}))} & \text{if } \phi\Pent \neg \alpha\\
G^+_\pr(\phi, \alpha) & \text{otherwise . }
\end{array}\right . 
\end{eqnarray*}

\end{corollary}
 \begin{proof}
By definition of $G^+_\pr$ we have 
\begin{eqnarray*}
G^+_\pr(\phi, \alpha) & = & \km(\phi + \alpha) - \km(\phi) \\     
& = & - \log_2 \pr(\phi \land \alpha) + \log_2 \pr(\phi) \\
& = & - \log_2 (\pr(\alpha\mid \phi) \cdot \pr(\phi)) + \log_2 \pr(\phi) \\
& = & - \log_2 \pr(\alpha\mid \phi) - \log_2 \pr(\phi) + \log_2 \pr(\phi) \\
& = & - \log_2 \pr(\alpha\mid \phi) \ .
\end{eqnarray*}

\nd By definition of $R^+_\pr$ we have the following two cases.
If $ \phi\not\Pent \neg \alpha$ then, by Proposition~\ref{revsimple}, ${\phi}^{\star}_{\alpha}$ is equivalent to 
$\phi \land \alpha$ and, thus, is equivalent to $\phi + \alpha$. Therefore, $R^\star_\pr(\phi, \alpha) = G^+_\pr(\phi, \alpha)$.
On the other hand, assume $\phi\Pent \neg \alpha$ instead. Then, by Proposition~\ref{revsimple}, 
${\phi}^{\star}_{\alpha} = \overl({\min_{\km}([\alpha]^+}))$ and, thus,
\begin{eqnarray*}
R^\star_\pr(\phi, \alpha)  & = & -\log_2 \pr(\overl({\min_{\km}([\alpha]^+}))) + \log_2 \pr(\phi)  \\
& = & \log_2 \frac{\pr(\phi)}{\pr(\overl({\min_{\km}([\alpha]^+}))} \ ,
\end{eqnarray*}

\nd which concludes.
 \end{proof}

\setcounter{proposition}{\getrefnumber{prop_iteratedC1}-1}
\begin{proposition}
Let $\star$ be a KM-based revision operator. Then $\star$ satisfies {\bf (C1), (C3)} and {\bf (C4)}.
\end{proposition}
\begin{proof}
We prove that each postulate holds.

\begin{description}
    \item[{\bf (C1).}]  If $\psi\Pent \alpha$, we have the following possible cases:

    \begin{description}
    \item[Case $\phi\not\Pent \neg\psi$ (and, necessarily,  $\phi\not\Pent \neg\alpha$).] Then, by Proposition~\ref{revsimple}, $[\phi^{\star}_{\alpha}] = [\phi]^+\cap[\alpha]^+$ and 
    $[(\phi^{\star}_{\alpha})^{\star}_{\psi}] =(([\phi]^+\cap[\alpha]^+)\cap[\psi]^+)$. 
    %
    %
    Now, $\phi\not\Pent \neg\psi$, by Proposition~\ref{revsimple}, $[\phi^{\star}_{\psi}]=[\phi]^+\cap[\psi]^+$. 
    Eventually, as by hypothesis $[\psi]^+\subseteq[\alpha]^+$,  $(([\phi]^+\cap[\alpha]^+)\cap[\psi]^+)=[\phi]^+\cap[\psi]^+$ follows, that is, $(\phi^{\star}_{\alpha})^{\star}_{\psi} \equiv_\pr \phi^{\star}_{\psi}$.

    \item[Case $\phi\Pent \neg\alpha$ (and, necessarily,  $\phi\Pent \neg\psi$).] Then, by Proposition~\ref{revsimple},  $[\phi^{\star}_{\alpha}] = 
    \min_{\km}([\alpha]^+)$. Since by hypothesis $[\psi]^+\subseteq[\alpha]^+$, we have two possibilities: \ii{i} $\min_{\km}([\alpha]^+\cap[\psi]^+) \neq\emptyset$, and in such a case $[(\phi^{\star}_{\alpha})^{\star}_{\psi}]=\min_{\km}([\alpha]^+)\cap[\psi]^+$ and, since $[\psi]^+\subseteq[\alpha]^+$, we have also $\min_{\km}([\alpha]^+)\cap[\psi]^+=\min_{\km}[\psi]^+$; \ii{ii} $\min_{\km}([\alpha]^+)\cap[\psi]^+=\emptyset$, and also in such a case $[(\phi^{\star}_{\alpha})^{\star}_{\psi}]=\min_{\km}([\psi]^+)$. On the other hand, given $\phi\Pent \neg\psi$ we have $[\phi^{\star}_{\psi}] = 
    \min_{\km}([\psi]^+)$. That is, $(\phi^{\star}_{\alpha})^{\star}_{\psi} \equiv_\pr \phi^{\star}_{\psi}$.
    
    \item[Case $\phi\Pent \neg\psi$ and $\phi\not\Pent \neg\alpha$.] Then, by Proposition~\ref{revsimple}, $[\phi^{\star}_{\alpha}]=[\phi]^+\cap[\alpha]^+$ and, since $[\phi]^+\subseteq[\neg \psi]^+$ by hypothesis, $[(\phi^{\star}_{\alpha})^{\star}_{\psi}]=\min_{\km}([\psi]^+)$. On the other hand, again since $[\phi]^+\subseteq[\neg \psi]^+$, $[\phi^{\star}_{\psi}]=\min_{\km}([\psi]^+)$. That is, $(\phi^{\star}_{\alpha})^{\star}_{\psi} \equiv_\pr \phi^{\star}_{\psi}$.
        
    \end{description}
    
\nd We can conclude that in any case $(\phi^{\star}_{\alpha})^{\star}_{\psi} \equiv_\pr \phi^{\star}_{\psi}$ holds, as desired.

\item[{\bf (C3).}]  If $\phi^{\star}_{\psi}\Pent \alpha$, we have the following possible cases:

\begin{description}
    \item[Case $\phi\Pent \neg\alpha$ (and, necessarily,  $\phi\Pent \neg\psi$).] Then, by Proposition~\ref{revsimple},  $[\phi^{\star}_{\alpha}]= 
    \min_{\km}([\alpha]^+)$. Since by hypothesis $[\psi]^+\subseteq[\alpha]^+$, we have two possibilities: 
    \ii{i} $\min_{\km}([\alpha]^+)\cap[\psi]^+\neq\emptyset$, and in such a case $[(\phi^{\star}_{\alpha})^{\star}_{\psi}]=\min_{\km}([\alpha]^+)\cap[\psi]^+$ and, since $[\psi]^+\subseteq[\alpha]^+$, we have also $\min_{\km}([\alpha]^+)\cap[\psi]^+=\min_{\km}([\psi]^+)$; 
    \ii{ii} $\min_{\km}([\alpha]^+(\cap[\psi]^+=\emptyset$, and also in such a case $[(\phi^{\star}_{\alpha})^{\star}_{\psi}]=\min_{\km}([\psi]^+)$. On the other hand, given $\phi\Pent \neg\psi$ we have $[\phi^{\star}_{\psi}]=\min_{\km}[\psi]^+$. That is, $(\phi^{\star}_{\alpha})^{\star}_{\psi} \equiv_\pr \phi^{\star}_{\psi}$.
    
   \item[Case $\phi\Pent\neg\alpha$.] Since $\phi^{\star}_{\psi}\Pent \alpha$, by Proposition~\ref{revsimple},  it must be $[\phi^{\star}_{\psi}]=\min_{\km}([\psi]^+)$. Also, by Proposition~\ref{revsimple}  we have $[\phi^{\star}_{\alpha}]=\min_{\km}([\alpha]^+)$. Now we have two possibilities: 
   \ii{i} $\min_{\km}([\alpha]^+)\subseteq [\neg \psi]^+$ and in such a case, by Proposition~\ref{revsimple}, $[(\phi^{\star}_{\alpha})^{\star}_{\psi}]=\min_{\km}([\psi]^+)$, that is, $[(\phi^{\star}_{\alpha})^{\star}_{\psi}]=[\phi^{\star}_{\psi}]$, and consequently  by hypothesis  $(\phi^{\star}_{\alpha})^{\star}_{\psi}\Pent \alpha$;
   \ii{ii} $\min_{\km}([\alpha]^+)\not\subseteq [\neg \psi]^+$ and in such a case $[(\phi^{\star}_{\alpha})^{\star}_{\psi}]=\min_{\km}([\alpha]^+)\cap [\psi]^+$, and consequently by hypothesis $(\phi^{\star}_{\alpha})^{\star}_{\psi}\Pent \alpha$.
   
   \item[Case $\phi\not\Pent\neg\alpha$.] Then, by Proposition~\ref{revsimple} $[\phi^{\star}_{\alpha}]=[\phi]^+\cap[\alpha]^+$. Now we have two possibilities: 
   \ii{i} $[\phi]^+\cap[\alpha]^+\subseteq [\neg \psi]^+$ and in such a case we have, by Proposition~\ref{revsimple},  $[(\phi^{\star}_{\alpha})^{\star}_{\psi}]=\min_{\km}([\psi]^+)$, that is, $[(\phi^{\star}_{\alpha})^{\star}_{\psi}]=[\phi^{\star}_{\psi}]$, and consequently by hypothesis $(\phi^{\star}_{\alpha})^{\star}_{\psi}\Pent \alpha$.
    \ii{ii} $[\phi]^+\cap[\alpha]^+\not \subseteq [\neg \psi]^+$ and in such a case, by Proposition~\ref{revsimple},   $[(\phi^{\star}_{\alpha})^{\star}_{\psi}]=[\phi]^+\cap[\alpha]^+\cap [\psi]^+$, and consequently $(\phi^{\star}_{\alpha})^{\star}_{\psi}\Pent \alpha$.
   
    \end{description}
    
\nd We can conclude that in any case $(\phi^{\star}_{\alpha})^{\star}_{\psi}\Pent \alpha$, as desired.

 \item[{\bf (C4).}]   If $\phi^{\star}_{\psi}\not\Pent \neg\alpha$, we have the following possible cases:

 \begin{description}
     \item[Case $\phi\Pent \neg \alpha$.] In such a case, by Proposition~\ref{revsimple}, $[\phi^{\star}_{\alpha}]=\min_{\km}([\alpha]^+)$. Now we have two possibilities: 
     \ii{i} $\min_{\km}([\alpha]^+)\subseteq [\neg \psi]^+$. In such a case, by Proposition~\ref{revsimple}, $[(\phi^{\star}_{\alpha})^{\star}_{\psi}]=\min_{\km}([\psi]^+)$, that is, $[(\phi^{\star}_{\alpha})^{\star}_{\psi}]=[\phi^{\star}_{\psi}]$, and consequently by hypothesis $(\phi^{\star}_{\alpha})^{\star}_{\psi}\not\Pent \neg\alpha$.
       \ii{ii} $\min_{\km}([\alpha]^+)\not\subseteq [\neg \psi]^+$. In such a case, by Proposition~\ref{revsimple},  $[(\phi^{\star}_{\alpha})^{\star}_{\psi}]=\min_{\km}([\alpha]^+)\cap [\psi]^+$, and consequently $(\phi^{\star}_{\alpha})^{\star}_{\psi}\not\Pent \neg\alpha$.
    
    \item[Case $\phi\not\Pent\neg\alpha$.] Then, by Proposition~\ref{revsimple},  $[\phi^{\star}_{\alpha}]=[\phi]^+\cap[\alpha]^+$. Now we have two possibilities: 
    \ii{i} $[\phi]^+\cap[\alpha]^+\subseteq [\neg \psi]^+$. 
    In such a case we have, by Proposition~\ref{revsimple},  $[(\phi^{\star}_{\alpha})^{\star}_{\psi}]=\min_{\km}([\psi]^+)$, that is, $[(\phi^{\star}_{\alpha})^{\star}_{\psi}]=[\phi^{\star}_{\psi}]$, and consequently by hypothesis $(\phi^{\star}_{\alpha})^{\star}_{\psi}\not\Pent \neg\alpha$.
        \ii{ii} $[\phi]^+\cap[\alpha]^+\not \subseteq [\neg \psi]^+$. In such a case $[(\phi^{\star}_{\alpha})^{\star}_{\psi}]=[\phi]^+\cap[\alpha]^+\cap [\psi]^+$, and consequently $(\phi^{\star}_{\alpha})^{\star}_{\psi}\not\Pent \neg \alpha$.
    
\end{description} 
We can conclude that in any case $(\phi^{\star}_{\alpha})^{\star}_{\psi}\not\Pent \neg\alpha$, as desired.
\end{description}
\nd This concludes the proof.
\end{proof}

\end{document}